\documentclass[10pt]{article}
\usepackage{mystyle}
\RequirePackage[top=2.25cm, bottom=2.5cm, left=2.5cm, right=2.5cm, columnsep=0.65cm]{geometry}
\usepackage{wrapfig}
\usepackage{etoc}
\IfFileExists{microtype.sty}{\usepackage{microtype}}{}
\allowdisplaybreaks[2]
\providecommand{\E}{\mathbb E}
\providecommand{\Prob}{\mathbb P}
\providecommand{\KL}{\operatorname{KL}}
\title{\LARGE When Does Momentum Help in Reinforcement Learning with Verifiable Rewards?}
\author{%
Kaifei Wang\thanks{\mbox{Peking University. Email: \texttt{\detokenize{{wkf5094,shanzikang}@stu.pku.edu.cn}}, \texttt{wanglw@pku.edu.cn}.}}\hspace{1.5em}
Zikang Shan\footnotemark[1]\hspace{1.5em}
Liwei Wang\footnotemark[1]\hspace{1.5em}
Yinyu Ye\thanks{\mbox{Shanghai Jiao Tong University, SIMIS, and Stanford University. Email: \texttt{yinyu-ye@stanford.edu}.}}\hspace{1.5em}
Han Zhong\thanks{\mbox{Shanghai Jiao Tong University. Email: \texttt{han.zhong@sjtu.edu.cn}.}}%
}
\date{}
\begin{document}
\etocdepthtag.toc{main}
\maketitle

\begin{abstract}
Momentum is a classical tool for accelerating convex optimization and a standard component of modern large language model training. 
Yet recent empirical work has questioned the usefulness of first-moment momentum in reinforcement learning with verifiable rewards (RLVR). 
We study this question through a bandit formulation of RLVR with bias-corrected exponential moving average (EMA) momentum, 
where the EMA coefficient $\beta$ controls how quickly the weights on past gradients decay. When policy gradients are evaluated exactly, 
we prove that increasing $\beta$ within a range that prevents overshoot strictly accelerates optimization. 
When gradients are estimated from sampled responses, a positive expected gain can be obscured by sampling fluctuations. 
We show that as the learning rate $\eta$ decreases, the expected gain with
fixed $\beta$ vanishes faster than these fluctuations, whereas increasing
$\beta$ to keep $\eta/(1-\beta)$ within the safe regime preserves a
nonvanishing gain that eventually dominates them.
Bandit simulations quantitatively validate these predictions.
RLVR experiments on Qwen3-4B-Base using SGD with EMA and AdamW support
our principle of choosing larger coefficients in practical LLM post-training.
On Qwen2.5-32B with AdamW, simply increasing $\beta_1$ improves mean
validation accuracy by nearly $2$ percentage points over the conventional
choice, demonstrating the principle's effectiveness at a larger model scale.
\end{abstract}

\section{Introduction}
Reinforcement learning with verifiable rewards (RLVR) has become an important approach to improving the reasoning capabilities of large language models (LLMs) \citep{lambert2024tulu,deepseek2025r1,yu2025dapo}. In RLVR, responses are sampled from the current policy and evaluated using verifiable rewards, so the training data distribution evolves as the policy changes. This differs from pretraining and supervised fine-tuning, which optimize next-token prediction (NTP) on given text sequences \citep{brown2020language,ouyang2022training3}. Despite these differences, RLVR commonly adopts Adam \citep{kingma2015adam} or AdamW \citep{loshchilov2019decoupled}, together with conventional first-moment momentum settings \citep{zhang2025adams,mukherjee2026adam}.
The design and configuration of optimizers for RLVR have received relatively limited attention.

One important aspect of optimizer design and configuration is the use of first-moment momentum, implemented in Adam and AdamW as an exponential moving average (EMA) of past gradients. While EMA has shown empirical benefits in NTP training
\citep{chen2023symbolic,zhao2025deconstructing},
\citet{mukherjee2026adam} report that momentum does not consistently improve
performance in their RLVR experiments and can even reduce it under a
conventional choice of the momentum coefficient. These observations motivate a closer examination of whether EMA
can benefit RLVR and what factors influence its effectiveness.

To understand when EMA benefits RLVR and how to choose its coefficient,
we adopt a KL-regularized response-level bandit formulation
\citep{xiong2024iterative}. We analyze natural policy gradient
(NPG, \citealp{kakade2001natural}) with a softmax policy parameterization
and bias-corrected EMA. NPG provides a tractable theoretical framework
for studying policy improvement while controlling changes in the policy
distribution, a principle shared by TRPO and PPO
\citep{schulman2015trpo,schulman2017ppo,shani2020adaptive}.
NPG-based analysis has also provided practical guidance for LLM
post-training \citep[e.g., ][]{gai2026demystifying}. Our contributions are as follows.

\vspace{5pt}
\noindent\textbf{Deterministic acceleration and stochastic gains.}
We first prove that, with exact gradients, increasing the EMA coefficient
within a safe regime strictly accelerates optimization, which appears to contrast with the limited benefits reported in prior RLVR experiments
\citep{mukherjee2026adam}. 
A deterministic gain, however, need not translate into a consistent
improvement when gradients are estimated from sampled responses,
as sampling fluctuations may obscure the benefit.
We characterize the expectation and variance of the objective gain
and prove that, for a fixed EMA coefficient, the expected gain
remains positive for sufficiently small learning rates but vanishes
faster than the sampling fluctuations as the learning rate decreases
(Theorems~\ref{thm:exact-benefit} and~\ref{cor:fixed-beta}).

\vspace{5pt}
\noindent\textbf{Principled selection of momentum coefficients.}
We next show that scaling the EMA coefficient
$\beta$ with the learning rate $\eta$ so that $\eta/(1-\beta)$ is
a positive constant within the safe regime yields a nonvanishing
expected gain that eventually dominates sampling fluctuations
(Theorem~\ref{thm:fixed-horizon}). We further prove that an appropriately
increased coefficient allows a smaller learning rate to achieve comparable expected performance under the same update and sample budgets,
with asymptotically lower variance of the objective value
(Theorem~\ref{thm:smaller-lr}).

\vspace{5pt}
\noindent\textbf{Validation in bandits and LLM post-training.}
We quantitatively validate the predicted gain and fluctuation scales
in numerical bandit experiments. GRPO experiments
\citep{shao2024deepseekmath} on Qwen3-4B-Base \citep{yang2025qwen3}
support our theoretical predictions in practical LLM post-training
with both SGD with EMA and the widely used AdamW optimizer: larger
coefficients improve performance and, when paired with smaller learning
rates, yield more stable training with comparable or slightly improved
performance. We further demonstrate the principle's effectiveness at
a larger model scale: on Qwen2.5-32B \citep{yang2024qwen25}, simply increasing AdamW's first-moment coefficient improves mean validation accuracy by nearly $2$ percentage points over the conventional choice.

\section{Related work}
\label{sec:related}

\noindent\textbf{Optimization and momentum in RLVR.}
Recent work revisits optimizer design for RLVR through momentum-based
normalization~\citep{zhang2025adams}, spectral
transformations~\citep{fan2026rethinking,shi2026pion,zhu2026iso}, and
parameter-space exploration~\citep{venkatkrishna2026exploration}.
Focusing on optimizer momentum, \citet{fontana2026hidden} analyze how
accumulated AdamW momentum can continue driving updates after clipping
suppresses the current policy-gradient contribution during repeated
optimization of the same batch.
Most directly related to our motivation, \citet{mukherjee2026adam} find
that SGD can match or outperform AdamW and report no consistent benefit
from momentum.
Motivated by these empirical findings, we study how the benefit of
ordinary EMA momentum in RLVR depends on the momentum coefficient,
learning rate, and sampling noise.

\vspace{5pt}
\noindent\textbf{Momentum in reinforcement learning.}
Prior work studies momentum for accelerating policy optimization and
reducing gradient-estimation error in conventional RL.
\citet{chen2024accelerated} establish accelerated convergence for
exact-gradient policy optimization under tabular softmax parameterization using
Nesterov momentum with a restart rule.
Momentum-based variance reduction has also been studied for policy
gradient~\citep{huang2020momentum,ding2022global} and natural policy
gradient~\citep{feng2024global}. \citet{huang2022bregman} develop Bregman
policy optimization with basic momentum and a variance-reduced variant.
This literature primarily concerns convergence and sample complexity in
conventional RL, with reported experiments on control and Atari benchmarks.
We study bias-corrected averaging of past natural gradients in a
KL-regularized RLVR bandit model, comparing objective values with and
without momentum and examining the practical implications through GRPO
experiments with language models.

\vspace{5pt}
\noindent\textbf{Momentum in supervised learning and stochastic optimization.}
Analyses of online learning and general stochastic optimization
characterize regimes in which momentum behaves similarly to SGD under
appropriate learning-rate scaling~\citep{yuan2016momentum,wang2024marginal}.
In language-model pretraining, AdEMAMix combines fast and slow gradient
averages to exploit longer gradient memory~\citep{pagliardini2025ademamix},
while ADana combines time-varying momentum schedules with a direct
gradient term and damping~\citep{ferbach2026logarithmic}.
These pretraining studies concern next-token prediction on given text
sequences. Our focus is momentum in RLVR, where reward feedback is
obtained from responses sampled from the evolving policy.

\section{Preliminaries}
\label{sec:setting}
In this section, we introduce the model, algorithms, and notation used throughout the paper, providing the foundation for our theoretical analysis.
\subsection{Model and parameterization}

\noindent\textbf{Basic Setup.}
Let $x\in \mathcal X$ be a given input prompt, and let $y\in \mathcal Y$ be a corresponding response generated by the language model. 
The model is parameterized as a policy $\pi$ characterized by $\theta$, from which responses are sampled $y\sim\pi_{\theta}(\cdot\mid x)$. 
Each prompt-response pair $(x,y)$ is evaluated by a verifier and receives a reward $r(x,y)$.
We consider the case of deterministic binary rewards: $r(x,y)$ is 1 if and only if the response $y$ is a correct answer to the prompt $x$, and 0 otherwise.
Let $\mathbf r\in \{0,1\}^{|\mathcal Y|}$ denote the vector of rewards for all possible responses to the prompt $x$.
We maximize the expected reward while penalizing divergence from a fixed reference policy $\pi_{\mathrm{ref}}$, using the objective
\begin{equation}
J(\theta) = \E_{y\sim \pi_\theta(\cdot\mid x)}[r(x,y)] - \gamma \KL(\pi_\theta(\cdot\mid x)\Vert \pi_{\mathrm{ref}}(\cdot\mid x)),
\label{eq:objective}
\end{equation}
where $\KL(\pi_\theta(\cdot\mid x)\Vert \pi_{\mathrm{ref}}(\cdot\mid x)):= \sum_{y \in \mathcal Y} \pi_\theta(y\mid x) \log (\pi_\theta(y\mid x)/\pi_{\mathrm{ref}}(y\mid x))$ is the Kullback-Leibler divergence between the policy $\pi_\theta(\cdot\mid x)$ and the reference policy, and $\gamma>0$ is a regularization coefficient.
Assume $\pi_{\mathrm{ref}}$ is a full-support policy.
To exclude degenerate cases, we assume that the initial policy is suboptimal and that the response set contains both correct and incorrect responses.
The objective $J(\theta)$ encourages higher expected reward while penalizing deviations from a fixed reference policy, such as one obtained through supervised fine-tuning.

\vspace{5pt}
\noindent\textbf{Policy parameterization.}
We adopt the standard softmax parameterization to model response generation.
Given a prompt $x$, we parameterize the conditional distribution over responses
using a logit vector $\theta\in\mathbb R^{|\mathcal Y|}$ as
$\pi_\theta(y\mid x)=e^{\theta(x,y)}/\sum_{y'\in\mathcal Y}e^{\theta(x,y')}$,
where $\theta(x,y)$ is the logit score for response $y$ given prompt $x$.
This response-level softmax
parameterization provides a tractable abstraction of autoregressive
generation in modern LLMs, whose token-level policies use softmax
distributions~\citep{ziegler2019fine,zhong2025dpo}.
\subsection{Training dynamics}

\noindent\textbf{Natural policy gradient.}
Our theoretical analysis focuses on the natural policy gradient (NPG) method, which 
preconditions the policy gradient using the Fisher information matrix \citep{kakade2001natural}.
Let $\theta_t$ denote the parameter vector at iteration $t$ and $\pi_{\theta_t}(\cdot\mid x)$ the induced policy. 
When policy gradients are evaluated exactly, the NPG update is given by:
\begin{equation}
 \theta_{t+1}=\theta_t+\eta \mathcal F^\dagger(\theta_t)\nabla J(\theta_t),\qquad
 \mathcal F(\theta)=\E_{y\sim\pi_\theta(\cdot\mid x)}\left[
 \nabla\log\pi_\theta(y\mid x)\nabla\log\pi_\theta(y\mid x)^\top\right].
 \label{eq:npg}
\end{equation}
Here $\eta>0$ represents the learning rate, $\mathcal F(\theta)$ is the Fisher information matrix, and $\mathcal F^\dagger(\theta)$ is its Moore--Penrose
pseudoinverse. When policy gradients are estimated from sampled responses, we replace the gradient $\nabla J(\cdot)$ with its unbiased estimator $\widehat{\nabla J}(\cdot)$, which uses a batch of sampled responses to approximate the expected reward, 
while the KL gradient and the Fisher information matrix are computed exactly. Specifically, let $\hat\theta_t$ be the parameter vector under the sampled NPG update. The update is given by:
\begin{equation}
 \hat\theta_{t+1}=\hat\theta_t+\eta \mathcal F^\dagger(\hat\theta_t)\widehat{\nabla J}(\hat\theta_t).
 \label{eq:sampled-npg}
\end{equation}
NPG not only enjoys favorable convergence guarantees under softmax
parameterization \citep{agarwal2021theory,cen2022fast}, but also connects
to modern RL algorithms used in LLM post-training. It can be viewed as
maximizing a linearized objective while controlling the KL divergence
between consecutive policies \citep{shani2020adaptive,cai2020provably,zhong2023theoretical},
sharing the principle of limiting policy changes with TRPO~\citep{schulman2015trpo}
and with PPO and GRPO, which are used in LLM post-training
\citep{schulman2017ppo,ouyang2022training3,shao2024deepseekmath}.
We study NPG to shed light on the training dynamics of RL algorithms
used in modern LLM post-training.

\begin{proposition}
\label{prop:softmax-npg}
Under the softmax parameterization, the exact and sampled NPG directions
are given by:
\begin{align}
 \mathcal F^\dagger(\theta)\nabla J(\theta)
 &=\Big(\mathbf I-\frac{\mathbf1\mathbf1^\top}{|\mathcal Y|}\Big)
   \left[\mathbf r-\gamma(\theta-\log\pi_{\mathrm{ref}}(\cdot\mid x))\right],
 \label{eq:exact-npg-direction}\\
 \mathcal F^\dagger(\hat \theta)\widehat{\nabla J}(\hat \theta)
 &=\Big(\mathbf I-\frac{\mathbf1\mathbf1^\top}{|\mathcal Y|}\Big)
   \Big[\frac1B\sum_{b=1}^B
   \frac{r(x,y_b)}{\pi_{\hat\theta}(y_b\mid x)}\mathbf e_{y_b}
   -\gamma(\hat\theta-\log\pi_{\mathrm{ref}}(\cdot\mid x))\Big].
 \label{eq:sampled-npg-direction-main}
\end{align}
Here $B$ denotes the batch size used to estimate the policy gradient,
and $y_1,\ldots,y_B\stackrel{\mathrm{iid}}\sim\pi_{\hat\theta}(\cdot\mid x)$ are the sampled responses.
Let $\mathbf I$ be the identity matrix, $\mathbf1$ the all-ones vector, and \(\mathbf e_y\) the one-hot vector corresponding to response $y$.
\end{proposition}

\noindent\textbf{Exponential Moving Average (EMA) momentum.}
In classical optimization, Polyak's heavy-ball momentum can accelerate convergence under suitable
curvature and parameter conditions~\citep{polyak1964some}, and global convergence guarantees
are available for smooth convex objectives~\citep{ghadimi2015global}.
EMA is its normalized form and is widely used in modern LLM training
optimizers, such as Adam and AdamW
\citep{kingma2015adam,loshchilov2019decoupled}.
EMA maintains a first-moment momentum $m_t$, where the EMA coefficient
$\beta\in(0,1)$ controls the contribution of the current gradient to
the parameter update. With exact gradients, the NPG update with EMA is given by
\begin{equation}
 m_{t+1}=\beta m_t+(1-\beta)\mathcal F^\dagger(\theta_t^{\beta})\nabla J(\theta_t^{\beta}),
 \qquad
 \theta_{t+1}^{\beta}=\theta_t^{\beta}+\eta\frac{m_{t+1}}{1-\beta^{t+1}}.
 \label{eq:npg-ema}
\end{equation}
The initial condition is $m_0=0$ and $\theta_t^{\beta}$ denotes the parameter vector updated by EMA with coefficient $\beta$ at iteration $t$. In the stochastic setting, we replace $\nabla J(\cdot)$ with its unbiased estimator and denote the corresponding parameter vector by $\hat \theta^{\beta}_t$.

Without EMA, the current natural gradient
is applied in full to the current parameter update. With EMA, a gradient's
contribution is spread across the current and subsequent updates, while
each update uses a weighted sum of the current and past natural gradients.
Thus, EMA can be viewed as a temporal reallocation of gradient contributions
to parameter updates. To isolate this reallocation from changes in the total
gradient weight, we normalize $m_t$ by $1-\beta^t$ for $t\ge1$ before
using it to update the parameters. This makes the gradient weights in each
update sum to one, following the first-moment bias correction used in Adam
and AdamW.

Let $\omega_\beta(t,t')$ denote the weight of the gradient at time $t'$ in the update at time $t$, under EMA with coefficient $\beta$.
For $\beta\in(0,1)$, \eqref{eq:npg-ema} can be rewritten as
\[
\theta_{t+1}^{\beta}=\theta_t^{\beta}+\eta\sum_{t'=0}^{t}\omega_\beta(t,t')\mathcal F^\dagger(\theta_{t'}^{\beta})\nabla J(\theta_{t'}^{\beta}),
\qquad
\omega_\beta(t,t')=\frac{(1-\beta)\beta^{t-t'}}{1-\beta^{t+1}}.
\]
The stochastic update uses the same weights.
In particular, we define $\omega_0(t,t')=0$ for all $t'\neq t$ and $\omega_0(t,t)=1$, which recovers the standard NPG update without EMA. 
For consistency, we use the same notation with \(\beta=0\) to denote the corresponding quantities for vanilla NPG without EMA.

\section{Deterministic case: acceleration with exact gradients}
\label{sec:exact}
In this section, we consider deterministic optimization with exact policy
gradients, where the natural-gradient direction is given by
\eqref{eq:exact-npg-direction}. Maximizing $J$ in~\eqref{eq:objective} yields the unique optimal policy
$\pi^\star(y\mid x)\propto\pi_{\mathrm{ref}}(y\mid x)\exp(r(x,y)/\gamma)$.
For an arbitrary initial parameter $\theta_0$, define
\[
 \theta^\star:=\log\pi_{\mathrm{ref}}(\cdot\mid x)+\frac{\mathbf r}{\gamma}
 +\frac{\mathbf1^\top\left[\theta_0-\log\pi_{\mathrm{ref}}(\cdot\mid x)
                              -\mathbf r/\gamma\right]\mathbf1}{|\mathcal Y|},
\]
which is one of the maximizers of $J$ and induces $\pi^\star$.
Here $\mathbf r$ is the reward vector and $|\mathcal Y|$ is the number of possible responses.
We compare the exact NPG iterates obtained with EMA coefficient $\beta$
against those obtained without EMA ($\beta=0$). Both runs are initialized at $\theta_0$ with
$m_0=0$ and use the same learning rate $\eta$ and number of updates
$T$.

Following the update in~\eqref{eq:npg-ema}, we recursively define
a family of functions $\{f_t:[0,1)\to\mathbb R\}_{t\ge0}$ by
$f_0(\beta)=1$ and
$f_{t+1}(\beta)=f_t(\beta)-\eta\gamma\sum_{t'=0}^t
\omega_\beta(t,t')f_{t'}(\beta)$ for $t\ge0$.
In particular, for vanilla NPG, $f_t(0)=(1-\eta\gamma)^t$.
The following lemma shows that these functions characterize the natural
policy gradient and the parameter error at each iteration.

\begin{lemma}[Exact dynamics]
\label{lem:exact-dynamics}
The exact NPG iterates in~\eqref{eq:npg-ema} satisfy, for every $t\ge0$,
\begin{align}
 \mathcal F^\dagger(\theta_t^\beta)\nabla J(\theta_t^\beta)
 =f_t(\beta)\mathcal F^\dagger(\theta_0)\nabla J(\theta_0),\qquad
 \theta_t^\beta-\theta^\star
 =f_t(\beta)(\theta_0-\theta^\star)
 .
\end{align}
\end{lemma}
The proof is in Appendix~\ref{app:exact}. Lemma~\ref{lem:exact-dynamics} shows that the exact natural gradients
at all iterations are collinear and that the parameter error
$\theta_t^\beta-\theta^\star$ remains a scalar multiple of
$\theta_0-\theta^\star$. As long as
$f_{t'}(\beta)>0$ for all $0\le t'\le t$, the iterates up to time $t$
do not overshoot $\theta^\star$. Over these iterations, the iterates 
\(\theta_{t'}^\beta\) approach \(\theta^\star\), so the natural gradients point in the same direction and decrease in magnitude.
EMA averages the current natural gradient with earlier, larger ones, which can accelerate optimization.
The following theorem gives a sufficient range that prevents
overshoot and establishes a strict benefit from increasing $\beta$. The proof is in Appendix~\ref{app:exact}.

\begin{theorem}[Strict acceleration from momentum]
\label{thm:exact-benefit}
Assume $0<\eta\gamma<1$ and $T\ge2$.
For any $\beta_1,\beta_2$ in the range $[0,((1-\eta\gamma)/(1+\eta\gamma))^2]$ with $\beta_1<\beta_2$, we have
$J(\theta_T^{\beta_2})>J(\theta_T^{\beta_1})$.
In particular, every positive $\beta$ in this range satisfies
$J(\theta_T^\beta)>J(\theta_T^0)$. 
\end{theorem}

Thus, within the stated range, EMA accelerates deterministic
optimization at the same learning rate and update count through a
temporal reallocation of natural gradients. This raises two questions:
\emph{Why can a strictly positive exact-gradient benefit fail to yield
a reliable improvement in sampled training? How can this benefit be
reliably realized and exploited in sampled training?}

\section{Stochastic case: sampling fluctuations and expected gain}
\label{sec:stochastic}
We now study the stochastic setting, where natural gradients are estimated
from sampled responses, introducing noise into the updates. To analyze
training dynamics at small learning rates, we fix the problem, the batch
size $B\ge1$, and the optimization time $\tau=\eta T>0$, and let
$\eta=\tau/T\to0$.
We compare $J(\widehat\theta_T^\beta)$ and
$J(\widehat\theta_T^0)$ obtained with and without EMA. The two runs share
the same settings, including the initial parameter $\theta_0$, zero
initial momentum, learning rate $\eta$, batch size $B$, and update count
$T$, while using independent samples even when $\beta=0$.
Expectations and variances are over training
randomness. The objective $J$ is evaluated exactly.

By Proposition~\ref{prop:softmax-npg}, the sampled
natural-gradient direction is conditionally unbiased. Taking expectations
in the EMA updates then gives $\E[\widehat\theta_t^\beta]=\theta_t^\beta$.
To characterize the expected objective, we must also account for
sampling fluctuations.
Let $\Sigma_T^\beta$ denote a covariance coefficient characterizing the
accumulated sampling noise along the exact trajectory $\theta_t^\beta$.
It is a nonnegative weighted sum of noise covariance matrices evaluated
along this trajectory, where the noise comes from estimating the reward
vector $\mathbf r$ using a single sampled response. This coefficient is 
a deterministic positive semidefinite matrix, independent of the realized 
samples. Its precise definition is given in Appendix~\ref{app:response}. 
Using this coefficient, the following lemma expresses the expected objective 
difference as the exact difference plus a sampling correction.

\begin{lemma}[Expected objective difference]
\label{lem:locality-main}
Fix $B\ge1$ and $\tau>0$. For any $\beta\in[0,1)$, as $\eta=\tau/T\to0$, the
expected objective difference satisfies
\begin{equation}
 \begin{aligned}
 \E\left[J(\widehat\theta_T^\beta)-J(\widehat\theta_T^0)\right]
 =J(\theta_T^\beta)-J(\theta_T^0)
 +\frac\eta{2B}\operatorname{tr}\left[
 \nabla^2J(\theta_T^\beta)\Sigma_T^\beta
 -\nabla^2J(\theta_T^0)\Sigma_T^0\right]+O((\eta/B)^{3/2}).
 \end{aligned}
 \label{eq:uniform-mean}
\end{equation}
The remainder is uniform in $\beta\in[0,1)$, and
$\sup_{\beta\in[0,1)}\|\Sigma_T^\beta\|_{\mathrm{op}}$ is bounded
independently of $\eta$.
\end{lemma}
Establishing the expansion in Lemma~\ref{lem:locality-main} requires
controlling sampling fluctuations throughout training. This is
challenging because the sampled natural-gradient direction weights
rewards by inverse sampling probabilities: a sampled response
with positive reward and very small probability can produce a large
noise term.
To control this noise, we construct an auxiliary parameter update
process whose distance from the exact iterates
is bounded. Since the exact trajectory is bounded,
the sampling probabilities in this auxiliary process stay bounded away
from zero, and its sampling noise is bounded. 
During training, the probability that the auxiliary process and the original sampled 
process differ is sufficiently small. Therefore, we can characterize the expected objective difference using the auxiliary process.
A proof sketch follows and Appendix~\ref{app:localization} gives the full proof.

\vspace{5pt}
\noindent\textbf{Proof sketch.}
Let $\varepsilon_t^\beta$ denote the estimation noise in the sampled
natural gradient at $\widehat\theta_t^\beta$. We construct an
auxiliary process driven by the same noise sequence, initialized at
$\widetilde\theta_0^\beta=\theta_0$ and defined by
\begin{align*}
 \widetilde\varepsilon_t^\beta
=\mathbf1\Big\{\max_{t'\le t}
       \|\widetilde\theta_{t'}^\beta-\theta_{t'}^\beta\|_\infty\le1\Big\}
       \varepsilon_t^\beta,\quad
 \widetilde\theta_{t+1}^\beta=\widetilde\theta_t^\beta
   +\eta\sum_{t'=0}^t\omega_\beta(t,t')
       \left[\mathcal F^\dagger(\widetilde\theta_{t'}^\beta)
       \nabla J(\widetilde\theta_{t'}^\beta)
       +\widetilde\varepsilon_{t'}^\beta\right].
\end{align*}
The two processes coincide through the first exit from the unit
neighborhood of the exact path. After this exit, the exact direction
and EMA memory continue to evolve, while no further sampling noise
enters the auxiliary process.
The construction gives
$\max_{t\le T}\|\widetilde\theta_t^\beta-\theta_t^\beta\|_\infty=O(1)$
for every sample path. Since the exact trajectory is bounded, this implies
$\inf_{t\le T,y\in\mathcal Y}\pi_{\widetilde\theta_t^\beta}(y\mid x)=\Omega(1)$
and $\max_{t<T}\|\widetilde\varepsilon_t^\beta\|_\infty=O(1)$.
These bounds allow us to estimate the covariance and third moment of the auxiliary parameter fluctuations. A second-order Taylor expansion, together with a comparison of the noise covariances along the auxiliary and exact paths, then yields
\begin{equation*}
 \E[J(\widetilde\theta_T^\beta)]
 =J(\theta_T^\beta)+\frac\eta{2B}\operatorname{tr}\left[
 \nabla^2J(\theta_T^\beta)\Sigma_T^\beta\right]
 +O((\eta/B)^{3/2}).
\end{equation*}
For each fixed update and coordinate, the weighted noise contributions
have zero conditional mean, so their partial sums form a martingale.
Using a concentration bound and a union bound, we show that the probability of path separation
is sufficiently small, uniformly in $\beta$. Since $J$ is bounded, we obtain
\begin{align*}
 \left|\E[J(\widehat\theta_T^\beta)]
       -\E[J(\widetilde\theta_T^\beta)]\right|
 &\le2\sup_{\theta}|J(\theta)|\,
 \Prob(\exists\,t\le T:\widehat\theta_t^\beta\ne\widetilde\theta_t^\beta)=o((\eta/B)^{3/2}).
\end{align*}
Combining these estimates and subtracting the expansions for $\beta$
and $0$ proves Lemma~\ref{lem:locality-main}. 

Applying the same analysis to $J^2$ also allows us to quantify the
fluctuations in the objective difference through its variance.
\begin{proposition}[Variance of the objective difference]
\label{prop:variance-main}
Under the conditions of Lemma~\ref{lem:locality-main}, the independent
runs satisfy the following expansion uniformly in $\beta\in[0,1)$:
\begin{equation}
 \begin{aligned}
 \operatorname{Var}\left[J(\widehat\theta_T^\beta)-J(\widehat\theta_T^0)\right]
 =\frac\eta B\left[
 \nabla J(\theta_T^\beta)^\top\Sigma_T^\beta\nabla J(\theta_T^\beta)
 +\nabla J(\theta_T^0)^\top\Sigma_T^0\nabla J(\theta_T^0)\right]
 +O((\eta/B)^{3/2}).
 \end{aligned}
 \label{eq:variance-main}
\end{equation}
Under the nondegeneracy condition,
specifically,
$\liminf_{\eta\to0}\nabla J(\theta_T^0)^\top\Sigma_T^0\nabla J(\theta_T^0)>0$,
uniformly over $\beta\in[0,1)$, we have
\begin{equation}
 \operatorname{Var}\big[J(\widehat\theta_T^\beta)-J(\widehat\theta_T^0)\big]
 =\Theta(\eta/B).
 \label{eq:variance-order-main}
\end{equation}
\end{proposition}
The proof is in Appendix~\ref{app:variance}. 
Proposition~\ref{prop:variance-main} shows that the variance of the
objective difference remains of order $\eta/B$ regardless of
the EMA coefficient. Lemma~\ref{lem:locality-main} shows that the
expected difference is characterized by the exact objective difference. 
Thus, when the exact gain is small
relative to the standard-deviation scale $\sqrt{\eta/B}$, training 
fluctuations make a consistent benefit from EMA difficult to observe.

The critical metric is $\rho=\eta/(1-\beta).$
Intuitively, EMA averages natural gradients over an effective
memory window of $1/(1-\beta)$ steps. Since each exact update moves
the parameters by $O(\eta)$, the averaged gradients span a
parameter-displacement scale of $O(\eta/(1-\beta))$, which controls
their difference from the current gradient. When this ratio is small,
gradients vary little within the window, so EMA behaves similarly to
the update without momentum and the difference vanishes. The next two subsections quantify the cases
of a vanishing or fixed positive ratio.

\subsection{A fixed EMA coefficient gives a vanishing expected gain}
\label{sec:fixed-beta}
\begin{theorem}[Fixed EMA coefficient]
\label{cor:fixed-beta}
\looseness=-1 Fix $\beta\in(0,1)$, $B\ge1$, $\tau>0$, and suppose that the
nondegeneracy condition in Proposition~\ref{prop:variance-main} holds.
As $\eta=\tau/T\to0$, the expected objective gain satisfies
\begin{equation}
 \E[J(\widehat\theta_T^\beta)-J(\widehat\theta_T^0)]
 =C(\tau,\gamma)\frac{\eta\beta}{1-\beta}+O(\eta^{3/2}),
 \label{eq:fixed-beta-mean}
\end{equation}
where $C(\tau,\gamma)$ is a positive constant depending only on
$\tau$, $\gamma$, and the fixed problem. It is independent of $\beta$.
The probability that EMA yields a higher objective value satisfies
\begin{equation}
 \lim_{\eta\to0}\Prob\big(J(\widehat\theta_T^\beta)>J(\widehat\theta_T^0)\big)
 =1/2.
 \label{eq:fixed-beta-comparison}
\end{equation}
\end{theorem}
The proof is in Appendix~\ref{app:fixed-beta}.
Theorem~\ref{cor:fixed-beta} shows that fixed $\beta$ gives a positive
expected gain for sufficiently small $\eta$, but the gain is only of
order $\eta$, which matches the variance order in
Proposition~\ref{prop:variance-main} with fixed batch size. Sampling fluctuations therefore
overwhelm the gain, making the two algorithms difficult to distinguish
in individual runs.
This offers a possible explanation for the limited benefits of momentum
reported in prior RLVR experiments \citep{mukherjee2026adam}.

\subsection{Scaling the EMA coefficient with the learning rate}
\label{sec:fixed-horizon}
Beyond Theorem~\ref{cor:fixed-beta}, we further show that $\eta/(1-\beta)\to0$
forces the expected gain to vanish, even when $\beta$ depends on $\eta$.
The full proof is given in Appendix~\ref{app:vanishing}.
To satisfy the sufficient condition for avoiding overshoot in
Theorem~\ref{thm:exact-benefit}, this ratio must remain bounded. 
Thus, a nonvanishing expected gain requires $\eta/(1-\beta)=\Theta(1)$.
For simplicity, we set $\beta$ so that $\eta/(1-\beta)=\rho$,
where $\rho>\eta$ is a positive constant. In Appendix~\ref{app:fixed-horizon},
we prove that $\rho\in(0,1/(4\gamma)]$ ensures the sufficient condition
in Theorem~\ref{thm:exact-benefit}.
The following theorem compares the objective values
of two independent runs with and without EMA in this setting.

\begin{theorem}[EMA coefficient scaling with the learning rate]
\label{thm:fixed-horizon}
Fix $B\ge1$, $\tau>0$, and $0<\rho\le1/(4\gamma)$, and suppose that the
nondegeneracy condition in Proposition~\ref{prop:variance-main} holds.
Set $\beta=1-\eta/\rho$. As $\eta=\tau/T\to0$, the expected objective
gain satisfies
\begin{equation}
 \E[J(\widehat\theta_T^\beta)-J(\widehat\theta_T^0)]
 =C(\tau,\gamma,\rho)\beta\rho+O(\eta),
 \label{eq:constant-horizon-moments}
\end{equation}
where $C(\tau,\gamma,\rho)$ is a positive constant depending only on
$\tau$, $\gamma$, $\rho$, and the fixed problem. It is independent of
$\eta$ and $\beta$. The probability that EMA yields a higher objective value satisfies
\begin{equation}
 \Prob\big(J(\widehat\theta_T^\beta)>J(\widehat\theta_T^0)\big)
 =1-O(\eta/B).
 \label{eq:constant-horizon-probability}
\end{equation}
\end{theorem}
\vspace{-5pt}
The proof is in Appendix~\ref{app:fixed-horizon}.
When $\rho=\eta/(1-\beta)$ is held at a fixed positive value in the safe
regime, Theorem~\ref{thm:fixed-horizon} gives a nonvanishing expected gain,
while Proposition~\ref{prop:variance-main} shows that the standard
deviation of the objective difference is of order $\sqrt{\eta/B}$.
As $\eta\to0$, the expected gain therefore dominates sampling fluctuations,
and EMA's advantage becomes consistently observable in the stochastic case.
These findings suggest that choosing $\beta$ substantially above the
conventional fixed value can yield greater benefits in practice.

\begin{figure}[!t]
 \centering
 \subcapcentertrue
 \subfigure[Gain and Noise.\label{fig:bandit-uniform-gain-noise}]{\includegraphics[width=0.24\textwidth]{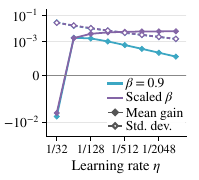}}\hfill
 \subfigure[Win Rate.\label{fig:bandit-uniform-win}]{\includegraphics[width=0.24\textwidth]{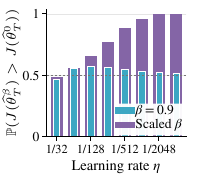}}\hfill
 \subfigure[Gain and Noise.\label{fig:bandit-random-gain-noise}]{\includegraphics[width=0.24\textwidth]{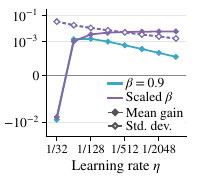}}\hfill
 \subfigure[Win Rate.\label{fig:bandit-random-win}]{\includegraphics[width=0.24\textwidth]{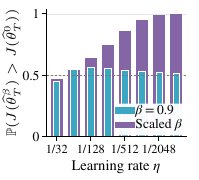}}
 \caption{Numerical bandit results with $3$ reward-one and $7$ reward-zero arms
 under uniform initialization (\ref{fig:bandit-uniform-gain-noise}, \ref{fig:bandit-uniform-win}) and random policy initialization (\ref{fig:bandit-random-gain-noise}, \ref{fig:bandit-random-win}).}
 \label{fig:bandit-results}
 \vspace{-10pt}
\end{figure}
\subsection{Reducing variability across runs with EMA}
\label{sec:smaller-lr}
We next investigate whether EMA can help reduce variability across runs while achieving comparable expected performance under the same training budget. A smaller learning rate can reduce sampling fluctuations, but it may also slow optimization over a given number of updates. The deterministic acceleration established in Section~\ref{sec:exact} suggests that increasing the EMA coefficient may compensate for this slower progress.
We compare two independent EMA runs with settings $(\eta_1,\beta_1)$ and $(\eta_2,\beta_2)$, using the same initialization, batch size $B$, and update count $T$. We write $\theta_T^{(i)}$ and $\widehat\theta_T^{(i)}$ for the parameters obtained with exact and sampled gradients, respectively, under setting $i=1,2$.
\begin{theorem}[Comparable expected performance with lower variance]
\label{thm:smaller-lr}
Fix $B>1$ and $\tau>0$.
For any $(\eta_1,\beta_1)$ with $\eta_1=\tau/T$ and $T\ge2$,
and any $\rho_2$ satisfying
$0<\rho_1=\eta_1/(1-\beta_1)<\rho_2\le1/(4\gamma)$,
there exist $0<\eta_2<\eta_1$ and $\beta_1<\beta_2<1$ such that
$\eta_2/(1-\beta_2)=\rho_2$ and $\theta_T^{(2)}=\theta_T^{(1)}$.
As $\eta_1\to0$ with $\rho_1,\rho_2$ fixed, the expected objective
difference and the limiting variance ratio satisfy
\begin{equation}
\E[J(\widehat\theta_T^{(2)})-J(\widehat\theta_T^{(1)})]
=O(\eta_1/B),
\qquad
\lim_{\tau\to\infty}\lim_{\eta_1\to0}
\frac{\operatorname{Var}[J(\widehat\theta_T^{(2)})]}
{\operatorname{Var}[J(\widehat\theta_T^{(1)})]}
=\frac{1+\sqrt{1-4\gamma\rho_2}}{1+\sqrt{1-4\gamma\rho_1}}<1.
\label{eq:smaller-lr-guarantee}
\end{equation}
\end{theorem}
\vspace{-3pt}
The proof is in Appendix~\ref{app:smaller-lr}. Increasing the EMA coefficient compensates for the smaller learning rate by matching the parameters obtained with exact gradients after the same number of updates. With sampled gradients, the expected objective values differ by $O(\eta_1/B)$, while the variance ratio converges to a constant below one in the stated asymptotic regime. Thus, EMA can help reduce variability across runs while achieving comparable expected performance under the same update count $T$ and sample budget $BT$.

\vspace{-5pt}
\section{Experiments}
\label{sec:experiments}
We use numerical bandit experiments to quantitatively validate the
predictions of
Proposition~\ref{prop:variance-main} and
Theorems~\ref{cor:fixed-beta} and~\ref{thm:fixed-horizon}.
Then we turn to RLVR experiments on Qwen3-4B-Base to test our
theoretical predictions in a practical LLM post-training setting.
Finally, we apply our principle of increasing the EMA coefficient to
Qwen2.5-32B to assess its applicability as model size increases.
\vspace{-2pt}
\subsection{Validating theoretical predictions in bandits}
\label{sec:numerical-bandit}
\noindent\textbf{Setup.}
We consider a 10-arm bandit in which $3$ arms have reward $1$ and the
remaining $7$ have reward $0$. The reference policy is set equal to
the initial policy, with KL penalty coefficient $\gamma=1$,
batch size $B=1$, and optimization time $\tau=1$.
Training uses $T=32,64,\ldots,4096$ updates with learning rate
$\eta=1/T$. We compare sampled NPG with fixed $\beta=0.9$ or
scaled $\beta=1-4\eta$ against the no-EMA baseline.
Each configuration uses $50{,}000$ independent training replicates,
with $J$ evaluated exactly after training, to estimate the
expectation and standard deviation of the objective
gap $J(\widehat\theta_T^\beta)-J(\widehat\theta_T^0)$ and the probability of outperforming an independent no-EMA run.
To measure the gain relative to these fluctuations, we compare SNR of the objective gap, defined as
$\operatorname{SNR}:=|\E[J(\widehat\theta_T^\beta)-J(\widehat\theta_T^0)]|/
\operatorname{Std}[J(\widehat\theta_T^\beta)-J(\widehat\theta_T^0)]$.
Both uniform and random policy initializations are considered.
Additional experiments with different reward configurations are provided in Appendix~\ref{app:additional-bandit}.

\vspace{5pt}
\noindent\textbf{Results.}
As the learning rate $\eta$ decreases, Figure~\ref{fig:bandit-results} shows that
the standard deviation scales as $\sqrt{\eta}$ for small $\eta$ under
both momentum schedules, as predicted by Proposition~\ref{prop:variance-main}.
With fixed $\beta=0.9$, the mean gain remains positive for small $\eta$
but scales as $\eta$, decreasing faster than the standard deviation.
With scaled $\beta=1-4\eta$, it approaches a positive constant, consistent with Theorems~\ref{cor:fixed-beta} and~\ref{thm:fixed-horizon}.
These trends hold under both uniform and random initialization.
At $\eta=1/4096$, fixed $\beta$ gives SNRs of 0.0453 and 0.0406
under uniform and random initialization, respectively, while scaled
$\beta$ gives 3.7647 and 3.3982. (More details are
reported in Table~\ref{tab:bandit-snr}.) The probability of EMA outperforming an
independent no-EMA run approaches $1/2$ with fixed $\beta$ and $1$
with scaled $\beta$.

\subsection{Practical LLM post-training}
\label{sec:llm-experiments}
\begin{figure}[!t]
 \centering
 \subcapcentertrue
 \makeatletter
 \let\llmOriginalSubcaption\@makesubfigurecaption
 \renewcommand{\@makesubfigurecaption}[2]{\hbox{\kern0.0216\textwidth\llmOriginalSubcaption{#1}{#2}\kern-0.0216\textwidth}}
 \makeatother
 \subfigure[SGD with EMA.\label{fig:llm-sgd-momentum}]{\includegraphics[width=0.24\textwidth]{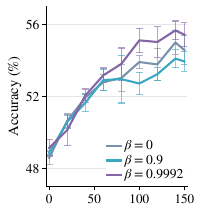}}\hfill
 \subfigure[AdamW.\label{fig:llm-adamw-momentum}]{\includegraphics[width=0.24\textwidth]{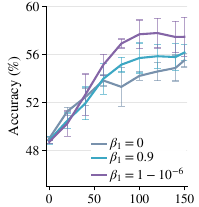}}\hfill
 \subfigure[SGD with EMA.\label{fig:llm-sgd-seeds}]{\includegraphics[width=0.24\textwidth]{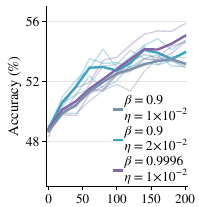}}\hfill
 \subfigure[AdamW.\label{fig:llm-adamw-seeds}]{\includegraphics[width=0.24\textwidth]{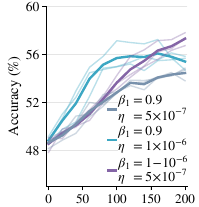}}
 \caption{RLVR experiments on Qwen3-4B-Base with GRPO. The horizontal axis shows the training step.
 \ref{fig:llm-sgd-momentum} and \ref{fig:llm-adamw-momentum} compare momentum
 coefficients at a fixed learning rate; lines show mean validation accuracy
 and error bars show $\pm1$ sample standard deviation across three seeds.
 \ref{fig:llm-sgd-seeds} and \ref{fig:llm-adamw-seeds} compare a smaller learning rate with a larger momentum coefficient
against a larger learning rate with the conventional momentum coefficient; light curves show individual runs and dark curves their means.}
 \label{fig:llm-4b-results}
 \vspace{-10pt}
\end{figure}

\noindent\textbf{Setup.}
We conduct RLVR experiments on
Qwen3-4B-Base with GRPO~\citep{shao2024deepseekmath}, using the DAPO-17k-processed dataset with $1{,}000$ problems held out for validation.
We vary $\beta$ in SGD with bias-corrected EMA and, to support our 
claim in a more practical setting, also study AdamW by varying $\beta_1$ 
while fixing $\beta_2=0.999$. Setting $\beta=0$ recovers vanilla SGD,
while $\beta_1=0$ gives an RMSProp-style optimizer.
All configurations share the same data and evaluation protocol.
We use three training generation seeds per configuration to show that our findings are consistent.
Details of the experimental setup are provided in Appendix~\ref{app:llm-experiments}.

\vspace{5pt}
\noindent\textbf{Benefits of larger momentum coefficients.}
For SGD with EMA, we compare $\beta\in\{0,0.9,0.9992\}$ at
$\eta=2\times10^{-2}$ (Figure~\ref{fig:llm-sgd-momentum}).
We choose $\beta=0.9992$ so that $\eta/(1-\beta)$ equals the upper bound
$1/(4\gamma)$ suggested by Theorem~\ref{thm:fixed-horizon}.
For AdamW, we compare $\beta_1\in\{0,0.9,1-10^{-6}\}$ at
$\eta=1\times10^{-6}$ (Figure~\ref{fig:llm-adamw-momentum}).
For both optimizers, the tuned larger momentum coefficient substantially
improves validation performance over the conventional value of $0.9$
and the baseline without first-moment momentum.
These results support our claim that choosing momentum coefficients
substantially above the conventional value of $0.9$ can yield greater
benefits in practical RLVR training.
\setlength{\intextsep}{0pt}
\begin{wrapfigure}{R}{0.5\textwidth}
\setlength{\abovecaptionskip}{-5pt}
\setlength{\belowcaptionskip}{0pt}
\centering
\includegraphics[width=\linewidth]{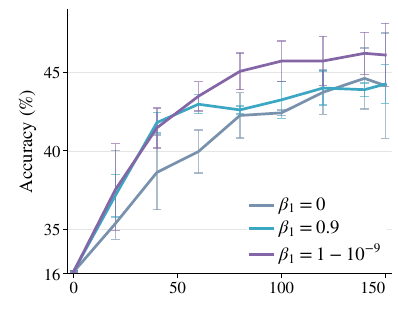}
\caption{RLVR experiment on Qwen2.5-32B with AdamW. The horizontal axis shows the training step. Lines show means and error bars show $\pm1$ sample standard deviation
across three seeds.}
\label{fig:32b-adamw}
\end{wrapfigure}

\vspace{5pt}
\noindent\textbf{More stable training without hurting performance.}
For SGD with EMA, we compare $(\eta,\beta)=(1\times10^{-2},0.9996)$
with $(2\times10^{-2},0.9)$ and include $(1\times10^{-2},0.9)$ as a baseline (Figure~\ref{fig:llm-sgd-seeds}).
Here, $\beta=0.9996$ is chosen to satisfy $\eta/(1-\beta)=1/(4\gamma)$. For AdamW, we compare $(\eta,\beta_1)=(5\times10^{-7},1-10^{-6})$
with $(1\times10^{-6},0.9)$ and include $(5\times10^{-7},0.9)$ as a
baseline (Figure~\ref{fig:llm-adamw-seeds}).
Under the same update and sample budgets, using a larger momentum
coefficient at the smaller learning rate preserves the 
low variability in final validation accuracy across seeds
while achieving slightly better performance than runs using the larger learning rate with conventional momentum.
These results show that pairing a larger momentum coefficient
with a smaller learning rate makes training more stable while
maintaining or even slightly improving performance.

\subsection{Extending to a larger model}
We apply the same GRPO setup to Qwen2.5-32B and increase the momentum coefficient in AdamW from $0.9$ to $1-10^{-9}$, using a learning rate of $2\times10^{-6}$.
Simply increasing the momentum coefficient $\beta_1$
substantially improves validation performance. At step~$150$, the conventional
$\beta_1=0.9$ achieves a mean accuracy of $44.3\%$, similar to
$44.2\%$ for the baseline without first-moment momentum, whereas
$\beta_1=1-10^{-9}$ reaches $46.1\%$, nearly 2 percentage points
higher than the conventional setting. Results are shown in Figure~\ref{fig:32b-adamw}. 

\section{Conclusion}
In this work, we explain why momentum's benefit can be obscured by
sampling noise and derive a principle for recovering and exploiting it.
We develop a theoretical model and prove that EMA strictly improves the
objective value under exact gradients within a range that prevents
overshoot. With sampled gradients, fixing $\beta$ as $\eta$ decreases
makes the expected gain vanish faster than sampling fluctuations,
whereas suitably increasing $\beta$ within the safe regime produces a
consistent improvement.
GRPO experiments on LLMs confirm the practical benefits of larger
momentum coefficients. These findings suggest using coefficients
substantially above the conventional $\beta=0.9$ to better exploit the
first-moment information already maintained by the optimizer.

\subsection*{AI use statement}
\looseness=-1 We used generative AI to assist with some mathematical derivations,
literature searches, proof checking, and language editing. The authors carefully reviewed all AI-assisted material and take full responsibility for this paper.

\bibliographystyle{ims}
\bibliography{references}
\clearpage
\appendix
\section{Proof of Proposition~\ref{prop:softmax-npg}}
\label{app:geometry}
\begin{proof}
We first identify the softmax geometry and then derive the exact and sampled natural gradients.

\vspace{5pt}
\noindent\textbf{Policy Jacobian and Fisher matrix.}
Differentiating the softmax parameterization, we have
\begin{equation}
 \frac{\partial\pi_\theta(y\mid x)}{\partial\theta(x,y')}
 =\pi_\theta(y\mid x)\big(\mathbf1_{\{y=y'\}}-\pi_\theta(y'\mid x)\big),\quad
  \nabla\log\pi_\theta(y\mid x)=\mathbf e_y-\pi_\theta(\cdot\mid x),
 \label{eq:softmax-coordinate}
\end{equation}
where $\mathbf e_y$ is the one-hot vector corresponding to $y\in\mathcal Y$.
Substituting the second identity into the definition of the Fisher
information matrix in~\eqref{eq:npg} and comparing with the first gives
\begin{equation}
 \mathcal F(\theta)=\operatorname{diag}\big(\pi_\theta(\cdot\mid x)\big)
 -\pi_\theta(\cdot\mid x)\pi_\theta(\cdot\mid x)^\top
 =\frac{\partial\pi_\theta(\cdot\mid x)}{\partial\theta^\top}.
 \label{eq:fisher-appendix}
\end{equation}
Since the policy has full support, $\mathcal F(\theta)$ is positive
semidefinite with null space $\operatorname{span}\{\mathbf1\}$. Hence
$\mathcal F^\dagger(\theta)\mathcal F(\theta)
=\mathbf I-\mathbf1\mathbf1^\top/|\mathcal Y|$.

\vspace{5pt}
\noindent\textbf{Exact and sampled natural gradients.}
By the chain rule, differentiating~\eqref{eq:objective} gives
\begin{align}
 \nabla J(\theta)
 &=\Big(\frac{\partial\pi_\theta(\cdot\mid x)}{\partial\theta^\top}\Big)^\top\mathbf r
 -\gamma\sum_{y\in\mathcal Y}\Big(
 \log\frac{\pi_\theta(y\mid x)}{\pi_{\mathrm{ref}}(y\mid x)}+1\Big)
 \nabla\pi_\theta(y\mid x)\notag\\
 &=\mathcal F(\theta)\left[\mathbf r-\gamma\left(
 \log\pi_\theta(\cdot\mid x)-\log\pi_{\mathrm{ref}}(\cdot\mid x)+\mathbf1
 \right)\right]\notag\\
 &=\mathcal F(\theta)
 \left[\mathbf r-\gamma\big(\theta-\log\pi_{\mathrm{ref}}(\cdot\mid x)\big)\right].
 \label{eq:objective-gradient-appendix}
\end{align}
The second equality follows from~\eqref{eq:fisher-appendix}.
The last uses
$\log\pi_\theta(\cdot\mid x)=\theta-
\log(\sum_{y\in\mathcal Y}\exp(\theta(x,y)))\mathbf1$
and $\mathcal F(\theta)\mathbf1=0$.
For the sampled gradient, the KL term is evaluated exactly, so
\begin{align}
 \widehat{\nabla J}(\widehat\theta)
 &=\frac1B\sum_{b=1}^B r(x,y_b)\nabla\log\pi_{\widehat\theta}(y_b\mid x)
   -\gamma\nabla\KL\big(\pi_{\widehat\theta}(\cdot\mid x)\Vert\pi_{\mathrm{ref}}(\cdot\mid x)\big)\notag\\
 &=\mathcal F(\widehat\theta)\Big[
   \frac1B\sum_{b=1}^B\frac{r(x,y_b)}{\pi_{\widehat\theta}(y_b\mid x)}\mathbf e_{y_b}
   -\gamma\big(\widehat\theta-\log\pi_{\mathrm{ref}}(\cdot\mid x)\big)\Big].
 \label{eq:sampled-objective-gradient}
\end{align}
The second equality uses
$\mathcal F(\widehat\theta)\mathbf e_{y_b}/\pi_{\widehat\theta}(y_b\mid x)
=\mathbf e_{y_b}-\pi_{\widehat\theta}(\cdot\mid x)$.
Premultiplying~\eqref{eq:objective-gradient-appendix}
and~\eqref{eq:sampled-objective-gradient} by the corresponding Fisher
pseudoinverse proves~\eqref{eq:exact-npg-direction}
and~\eqref{eq:sampled-npg-direction-main}.
\end{proof}
\section{Proofs of Lemma~\ref{lem:exact-dynamics} and Theorem~\ref{thm:exact-benefit}}
\label{app:exact}
\noindent\textbf{Optimality of $\theta^\star$.}
The definition of $\theta^\star$ gives
$\pi_{\theta^\star}(y\mid x)\propto\pi_{\mathrm{ref}}(y\mid x)e^{r(x,y)/\gamma}$.
Substituting this identity into~\eqref{eq:objective}, we obtain
\begin{equation}
 J(\theta^\star)-J(\theta)
 =\gamma\KL\big(\pi_\theta(\cdot\mid x)\Vert\pi_{\theta^\star}(\cdot\mid x)\big)\ge0.
 \label{eq:objective-gap-identity}
\end{equation}
Equality holds exactly when $\pi_\theta=\pi_{\theta^\star}$, which
verifies the unique optimal policy used in Section~\ref{sec:exact}.
\begin{proof}[Proof of Lemma~\ref{lem:exact-dynamics}]
Substituting the definition of $\theta^\star$ into
\eqref{eq:exact-npg-direction} gives
\begin{equation*}
 \mathcal F^\dagger(\theta_t^\beta)\nabla J(\theta_t^\beta)
 =\gamma\bigg(\mathbf I-\frac{\mathbf1\mathbf1^\top}{|\mathcal Y|}\bigg)
 (\theta^\star-\theta_t^\beta).
\end{equation*}
By Proposition~\ref{prop:softmax-npg}, every natural-gradient direction
has coordinate sum zero. Starting from $m_0=0$, the EMA update therefore
gives $\mathbf1^\top m_t=0$ at every step, so the parameter update preserves
$\mathbf1^\top\theta_t^\beta=\mathbf1^\top\theta_0$.
By construction, $\mathbf1^\top\theta^\star=\mathbf1^\top\theta_0$ as well.
Thus $\mathbf1^\top(\theta^\star-\theta_t^\beta)=0$, and the preceding identity
simplifies to
\begin{equation}
 \mathcal F^\dagger(\theta_t^\beta)\nabla J(\theta_t^\beta)
 =\gamma(\theta^\star-\theta_t^\beta).
 \label{eq:exact-direction-difference}
\end{equation}

\noindent\textbf{Scalar residual dynamics.}
We prove the parameter identity by induction. It holds at $t=0$
since $f_0(\beta)=1$. If it holds through iteration $t$, unrolling
\eqref{eq:npg-ema} yields
\begin{align*}
 \theta_{t+1}^\beta-\theta^\star
 &=\theta_t^\beta-\theta^\star-\eta\gamma
   \sum_{t'=0}^t\omega_\beta(t,t')(\theta_{t'}^\beta-\theta^\star)\\
 &=\Big[f_t(\beta)-\eta\gamma\sum_{t'=0}^t\omega_\beta(t,t')f_{t'}(\beta)\Big]
   (\theta_0-\theta^\star)
 =f_{t+1}(\beta)(\theta_0-\theta^\star).
\end{align*}
The last equality is the defining recurrence of $f_t$.
Substituting the parameter identity into~\eqref{eq:exact-direction-difference}
proves the natural-gradient identity.
\end{proof}
\begin{proof}[Proof of Theorem~\ref{thm:exact-benefit}]
\looseness=-1 We first establish positivity of $f_t(\beta)$ in the safe range and show that it decreases with the EMA coefficient for $t\ge2$. We then use this residual comparison to prove strict improvement of the objective.

\vspace{5pt}
\noindent\textbf{Positive residuals in the safe range.}
We first prove that $f_t(\beta)>0$ for all $t\ge0$ when
$0<\eta\gamma<1$ and $0\le\beta\le((1-\eta\gamma)/(1+\eta\gamma))^2$.
{At $\beta=0$, $f_t(0)=(1-\eta\gamma)^t>0$, so it remains to consider $\beta>0$.}
Separating the current term in the defining recurrence of $f_t$ and
substituting the preceding update gives, for $t\ge1$,
\begin{align}
 f_{t+1}(\beta)
 &=f_t(\beta)-\frac{\eta\gamma}{\sum_{t'=0}^t\beta^{t'}}
 \Big[f_t(\beta)+\beta\sum_{t'=0}^{t-1}\beta^{t-1-t'}f_{t'}(\beta)\Big]\notag\\
 &=f_t(\beta)-\frac{\eta\gamma}{\sum_{t'=0}^t\beta^{t'}}f_t(\beta)
 +\frac{\beta\sum_{t'=0}^{t-1}\beta^{t'}}{\sum_{t'=0}^t\beta^{t'}}
 \big[f_t(\beta)-f_{t-1}(\beta)\big]\notag\\
 &=(1-\eta\gamma)f_t(\beta)
 +\frac{\beta\sum_{t'=0}^{t-1}\beta^{t'}}{\sum_{t'=0}^t\beta^{t'}}
 \big[(1+\eta\gamma)f_t(\beta)-f_{t-1}(\beta)\big].
 \label{eq:second-order-f}
\end{align}
For $\beta>0$, we prove positivity and
$\sqrt\beta\le f_{t+1}(\beta)/f_{t}(\beta)\le1-\eta\gamma$
simultaneously by induction. Division gives
\begin{equation}
 \frac{f_{t+1}(\beta)}{f_t(\beta)}=1-\eta\gamma
 +\frac{\beta\sum_{t'=0}^{t-1}\beta^{t'}}{\sum_{t'=0}^t\beta^{t'}}
 \bigg(1+\eta\gamma-\frac{f_{t-1}(\beta)}{f_t(\beta)}\bigg).
 \label{eq:exact-ratio}
\end{equation}
The base case holds because $f_0(\beta)=1$,
$f_1(\beta)=1-\eta\gamma>0$, and
$f_1(\beta)/f_0(\beta)=1-\eta\gamma\ge\sqrt\beta$.
Assume $f_t(\beta)>0$ and $\sqrt\beta\le f_t(\beta)/f_{t-1}(\beta)\le1-\eta\gamma$ for some $t\ge1$. The coefficient satisfies
$0\le\beta\sum_{t'=0}^{t-1}\beta^{t'}/\sum_{t'=0}^t\beta^{t'}<\beta$,
and the upper bound on $f_t(\beta)/f_{t-1}(\beta)$ gives
$1+\eta\gamma-f_{t-1}(\beta)/f_t(\beta)
\le-(\eta\gamma)^2/(1-\eta\gamma)<0$.
Thus~\eqref{eq:exact-ratio} gives $f_{t+1}(\beta)/f_{t}(\beta)<1-\eta\gamma$. For the lower bound,
\begin{align*}
 \frac{f_{t+1}(\beta)}{f_t(\beta)}-\sqrt\beta
 &\ge1-\eta\gamma+\beta\bigg(1+\eta\gamma-\frac1{\sqrt\beta}\bigg)-\sqrt\beta\\
 &=(1-\sqrt\beta)\big[(1-\eta\gamma)-(1+\eta\gamma)\sqrt\beta\big]\ge0.
\end{align*}
The last inequality follows from the assumed range of $\beta$.
Consequently $f_{t+1}(\beta)\ge\sqrt\beta f_t(\beta)>0$.

\vspace{5pt}
\noindent\textbf{Strict decrease with the EMA coefficient.}
Differentiating~\eqref{eq:exact-ratio} with respect to $\beta$ gives
\begin{align*}
 \partial_\beta\bigg(\frac{f_{t+1}(\beta)}{f_t(\beta)}\bigg)
 &=\frac{\beta\sum_{t'=0}^{t-1}\beta^{t'}}{\sum_{t'=0}^t\beta^{t'}}
   \bigg(\frac{f_{t-1}(\beta)}{f_t(\beta)}\bigg)^2
   \partial_\beta\bigg(\frac{f_t(\beta)}{f_{t-1}(\beta)}\bigg)+\frac{\sum_{t'=1}^t t'\beta^{t'-1}}{(\sum_{t'=0}^t\beta^{t'})^2}
   \bigg(1+\eta\gamma-\frac{f_{t-1}(\beta)}{f_t(\beta)}\bigg).
\end{align*}
The first term has the same sign as
$\partial_\beta(f_t(\beta)/f_{t-1}(\beta))$
for $\beta>0$.
The second term is strictly negative for $t\ge1$. Starting from
$\partial_\beta(f_1(\beta)/f_0(\beta))=\partial_\beta(1-\eta\gamma)=0$, induction therefore gives
$\partial_\beta(f_{t+1}(\beta)/f_t(\beta))<0$. Differentiating
$f_T(\beta)=\prod_{t=1}^T f_{t}(\beta)/f_{t-1}(\beta)$ gives, for all $T\geq 2$,
\[
 \partial_\beta f_T(\beta)
 =f_T(\beta)\sum_{t=1}^T\frac{f_{t-1}(\beta)}{f_{t}(\beta)}
   \partial_\beta\bigg(\frac{f_{t}(\beta)}{f_{t-1}(\beta)}\bigg)<0.
\]
Thus, for $0<\beta_1<\beta_2$, we have $0<f_T(\beta_2)<f_T(\beta_1)<(1-\eta\gamma)^T=f_T(0)$.
The last inequality follows from multiplying the upper bound of $f_{t+1}(\beta)/f_t(\beta)$.

\vspace{5pt}
\noindent\textbf{Strict improvement of the objective.}
We now convert the residual comparison into an objective comparison.
By~\eqref{eq:objective-gradient-appendix}, for $a>0$,
\begin{equation}
 \frac{\mathrm d}{\mathrm d a}
 J(\theta^\star+a(\theta_0-\theta^\star))
 =-\gamma a\operatorname{Var}_{y\sim\pi_{\theta^\star+a(\theta_0-\theta^\star)}(\cdot\mid x)}
 [\theta_0(x,y)-\theta^\star(x,y)]<0.
 \label{eq:objective-residual-derivative}
\end{equation}
The variance is positive since the policy has full support and
$\theta_0\neq\theta^\star$.
Lemma~\ref{lem:exact-dynamics} and
$0<f_T(\beta_2)<f_T(\beta_1)$ therefore give
$J(\theta_T^{\beta_2})>J(\theta_T^{\beta_1})$.
\end{proof}
\section{Proof of Lemma~\ref{lem:locality-main}}
\label{app:localization}
\label{app:response}
\begin{proof}
We first construct an auxiliary process with bounded noise and control its exit probability and moments. We then use these estimates to derive the objective expansion for the auxiliary process and transfer it to the original process.

\vspace{5pt}
\noindent\textbf{Sampling noise and its propagation.}
Let $y_{t,b}^\beta$ be the $b$-th response at update $t$.
Conditional on the sample history, the batch consists of independent responses from
$\pi_{\widehat\theta_t^\beta}(\cdot\mid x)$.
Let $\varepsilon_{t,b}^\beta$ denote the estimation noise from the $b$-th
response at update $t$. Then we have
\begin{equation}
 \varepsilon_{t,b}^\beta=
 \left(\mathbf I-\frac{\mathbf1\mathbf1^\top}{|\mathcal Y|}\right)
 \left[\frac{r(x,y_{t,b}^\beta)}
 {\pi_{\widehat\theta_t^\beta}(y_{t,b}^\beta\mid x)}\mathbf e_{y_{t,b}^\beta}
 -\mathbf r\right], \qquad
 \varepsilon_t^\beta=\frac1B\sum_{b=1}^B\varepsilon_{t,b}^\beta.
 \label{eq:innovation}
\end{equation}
Since the sampled directions have zero coordinate sum~\eqref{eq:sampled-npg-direction-main} and the initial
momentum is zero, the updates preserve
$\mathbf1^\top\widehat\theta_t^\beta
=\mathbf1^\top\theta_0=\mathbf1^\top\theta^\star$.
Therefore, the sampled direction at update $t$ in Proposition~\ref{prop:softmax-npg}
can be written as
$\gamma(\theta^\star-\widehat\theta_t^\beta)+\varepsilon_t^\beta$.
Each $\varepsilon_{t,b}^\beta$ has conditional mean zero, so the expected
direction is $\gamma(\theta^\star-\E[\widehat\theta_t^\beta])$.
Thus, after taking expectations, the sampled EMA updates obey the same
recurrence as the exact updates. With the same initial parameter and
zero initial momentum, this gives
$\E[\widehat\theta_t^\beta]=\theta_t^\beta$ by induction.

Let $\Psi_\beta(t,t')$ denote the weight of the noise
$\varepsilon_{t'}^\beta$ in the difference between the sampled and exact parameters,
so that $\widehat\theta_t^\beta-\theta_t^\beta =\sum_{t'=0}^{t-1}\Psi_\beta(t,t')\varepsilon_{t'}^\beta.$
Subtracting the exact EMA update from the sampled EMA update gives
\begin{equation}
 \widehat\theta_{t+1}^\beta-\theta_{t+1}^\beta
 =\widehat\theta_t^\beta-\theta_t^\beta+\eta\sum_{t'=0}^t\omega_\beta(t,t')
 \big[-\gamma(\widehat\theta_{t'}^\beta-\theta_{t'}^\beta)
       +\varepsilon_{t'}^\beta\big].
\label{eq:noise-expansion}
\end{equation}
The noise at update $t'$ first affects the parameter at time $t'+1$,
so $\Psi_\beta(t,t')=0$ for $t\le t'$.
Substituting the noise expansion above into~\eqref{eq:noise-expansion} gives for all $t\ge t'$
\begin{equation}
 \Psi_\beta(t+1,t')=\Psi_\beta(t,t')
 -\gamma\eta\sum_{k=t'}^t\omega_\beta(t,k)\Psi_\beta(k,t')
 +\eta\omega_\beta(t,t').
 \label{eq:response}
\end{equation}
The last term captures the contribution of $\varepsilon_{t'}^\beta$ retained
in the EMA memory. The sum captures how this noise changes later parameter
values and hence the exact natural-gradient directions evaluated at those values.
According to \eqref{eq:response}, $\Psi_\beta(t,t')$ is a deterministic function of $\eta,\gamma,\beta,t,t'$ and does not depend on the realized samples.
The covariance coefficient in Lemma~\ref{lem:locality-main} is
\begin{equation}
 \Sigma_T^\beta=\frac1\eta\sum_{t=0}^{T-1}\Psi_\beta(T,t)^2
\bigg[
\operatorname{diag}\bigg(
\frac{\mathbf r}{\pi_{\theta_{t}^\beta}(\cdot\mid x)}
\bigg)-\mathbf r\mathbf r^\top
\bigg].
 \label{eq:discrete-covariance}
\end{equation}
The division inside $\operatorname{diag}$ is elementwise.
The bracket is the covariance of
$r(x,y)\mathbf e_y/\pi_{\theta_{t}^\beta}(y\mid x)$ under
$y\sim\pi_{\theta_{t}^\beta}(\cdot\mid x)$, using $r(x,y)^2=r(x,y)$.
Thus $\Sigma_T^\beta$ is deterministic and positive semidefinite. The next lemma states the bounds on
$\Psi_\beta(t,t')$ and $f_t(\beta)$ needed in the subsequent proofs. The proof is delayed to Section~\ref{app:response-bounds}.
\begin{lemma}[Uniform response bounds]
\label{lem:response}
Let $f_t$ be defined in Section~\ref{sec:exact} and $\Psi_\beta(t,t')$ be defined by~\eqref{eq:response}. For $\beta\in[0,1)$ and integers $t\ge0$ with $t\eta\le\tau$, we have $|f_t(\beta)|\le e^{\gamma\tau}$.
The response coefficients satisfy the following bounds, where the first inequality holds for all $t'=0,\ldots,t-1$.
\begin{equation}
  |\Psi_\beta(t,t')|\le e^{\gamma\tau}\eta
 \bigg(1+\sum_{k=t'+1}^t\frac1k\bigg),\quad
 \sum_{t'=0}^{t-1}|\Psi_\beta(t,t')|\le e^{\gamma\tau}\tau,\quad
 \sum_{t'=0}^{t-1}\Psi_\beta(t,t')^2\le5e^{2\gamma\tau}t\eta^2.
 \label{eq:response-bounds}
\end{equation}
\end{lemma}
\noindent\textbf{The auxiliary process.}
We construct a process with bounded noise that agrees with the sampled process until it leaves a unit neighborhood of the exact path.
Use the same noise sequence as the original process and initialize
$\widetilde\theta_0^\beta=\theta_0$. As in the main text, define
\begin{equation}
 \begin{gathered}
 \widetilde\varepsilon_t^\beta
 =\mathbf1\bigg\{\max_{t'\le t}
 \|\widetilde\theta_{t'}^\beta-\theta_{t'}^\beta\|_\infty\le1\bigg\}
 \varepsilon_t^\beta,\qquad
 \widetilde\theta_{t+1}^\beta=\widetilde\theta_t^\beta
 +\eta\sum_{t'=0}^t\omega_\beta(t,t')
 \big[\mathcal F^\dagger(\widetilde\theta_{t'}^\beta)
 \nabla J(\widetilde\theta_{t'}^\beta)
 +\widetilde\varepsilon_{t'}^\beta\big].
 \end{gathered}
 \label{eq:stopping}
\end{equation}
For the bounds below, let
$\zeta^\beta=\inf\{t\ge1:
\|\widetilde\theta_t^\beta-\theta_t^\beta\|_\infty>1\}$
denote the first exit time, taking $\zeta^\beta=\infty$
if no exit occurs.
The two processes agree before exit; afterward, only fresh noise is suppressed. Since the stopping indicator is known before sampling, $\widetilde\varepsilon_t^\beta$ has conditional mean zero, and taking expectations in the auxiliary updates yields $\E[\widetilde\theta_t^\beta]=\theta_t^\beta$.

The bounds on probabilities and noise can be made explicit. Define
$M=\|\theta^\star\|_\infty
+e^{\gamma\tau}\|\theta_0-\theta^\star\|_\infty+1$
and $L=2(|\mathcal Y|e^{2M}+1)$.
Lemma~\ref{lem:exact-dynamics} gives
$\theta_t^\beta=\theta^\star+f_t(\beta)(\theta_0-\theta^\star)$,
and the bound $|f_t(\beta)|\le e^{\gamma\tau}$ in Lemma~\ref{lem:response} holds for
$t\le T$. On $\{t<\zeta^\beta\}$, the two processes agree and
$\|\widetilde\theta_t^\beta-\theta_t^\beta\|_\infty\le1$. Therefore,
\begin{equation*}
  \|\widetilde\theta_t^\beta\|_\infty
\le\|\theta_t^\beta\|_\infty+1
 \le\|\theta^\star\|_\infty
   +|f_t(\beta)|\cdot\|\theta_0-\theta^\star\|_\infty+1
 \le M.
\end{equation*}
Then we have $\pi_{\widetilde\theta_t^\beta}(y\mid x)
 \ge1/(|\mathcal Y|\exp(2M))$. According to the definition of the sampling noise in~\eqref{eq:innovation}, on $\{t<\zeta^\beta\}$, we have
$\|\varepsilon_{t,b}^\beta\|_\infty
\le2\big(1/\pi_{\widetilde\theta_t^\beta}(y_{t,b}^\beta\mid x)+1\big)
\le2\big(|\mathcal Y|e^{2M}+1\big)=L$. Applying~\eqref{eq:response} to the auxiliary
process and using the second inequality in~\eqref{eq:response-bounds} gives for every sample path
\begin{equation}
 \max_{t\le T}\|\widetilde\theta_t^\beta-\theta_t^\beta\|_\infty
 \le\max_{t\le T}\sum_{t'=0}^{t-1}|\Psi_\beta(t,t')|\cdot
 \|\widetilde\varepsilon_{t'}^\beta\|_\infty
 \le L\tau e^{\gamma\tau}.
 \label{eq:aux-compact}
\end{equation}
In particular, the auxiliary parameters remain in a fixed bounded region even after exit.

\vspace{5pt}
\noindent\textbf{Exit probability.}
We bound the probability that the auxiliary process exits the unit neighborhood of the exact trajectory by time $T$. Fix $1\le t\le T$. The first coordinate difference satisfies $(\widetilde\theta_t^\beta)_1-(\theta_t^\beta)_1=\sum_{t'=0}^{t-1}\sum_{b=1}^B\Psi_\beta(t,t')\mathbf1_{\{t'<\zeta^\beta\}}(\varepsilon_{t',b}^\beta)_1/B$.
The $tB$ summands, ordered by $t'$ and then $b$, have zero conditional mean given the preceding samples, so their partial sums form a martingale. Azuma-Hoeffding's inequality therefore gives, for every $u>0$,
\begin{equation}
 \Prob\big(|(\widetilde\theta_t^\beta)_1-(\theta_t^\beta)_1|>u\big)
 \le 2\exp\bigg[-\frac{u^2}{2\sum_{t'=0}^{t-1}\sum_{b=1}^B
 (L|\Psi_\beta(t,t')|/B)^2}\bigg]
 \le2\exp\bigg[-\frac{Bu^2}
 {10\tau L^2e^{2\gamma\tau}\eta}\bigg].
 \label{eq:coordinate-tail}
\end{equation}
The last inequality uses the third inequality in~\eqref{eq:response-bounds}. The same bound holds for every coordinate. The two processes agree through the update causing the first exit, so taking $u=1$ and a union bound over updates and coordinates gives
\begin{equation}
 \begin{aligned}
 \Prob(\exists\,t\le T:\widehat\theta_t^\beta\ne\widetilde\theta_t^\beta)
 &\le\Prob(\zeta^\beta\le T)\le2|\mathcal Y|T\exp(-B/(
 {10\tau L^2e^{2\gamma\tau}\eta})).
 \end{aligned}
 \label{eq:original-path-tail}
\end{equation}

\noindent\textbf{Moment bounds.}
We next control the second and third moments. For the second moment, we have
\begin{align}
 \E\|\widetilde\theta_t^\beta-\theta_t^\beta\|_2^2
 =\frac1{B^2}\sum_{t'=0}^{t-1}\sum_{b=1}^B\Psi_\beta(t,t')^2
 \E\left[\mathbf1_{\{t'<\zeta^\beta\}}
 \|\varepsilon_{t',b}^\beta\|_2^2\right]
 \le5\tau e^{2\gamma\tau}|\mathcal Y|L^2\frac\eta B.
 \label{eq:aux-second}
\end{align}
In the equality, cross terms vanish in expectation since the noise has conditional mean zero. The inequality uses the third inequality in~\eqref{eq:response-bounds} and the bound $\|\varepsilon_{t',b}^\beta\|_2^2\le|\mathcal Y|L^2$ on $\{t'<\zeta^\beta\}$.
For the third moment, the tail-integral formula gives $\E\|\widetilde\theta_t^\beta-\theta_t^\beta\|_2^3
 =3\int_0^\infty u^2
 \Prob\big(\|\widetilde\theta_t^\beta-\theta_t^\beta\|_2>u\big)
 \,\mathrm d u.$
If $\|\widetilde\theta_t^\beta-\theta_t^\beta\|_2$ exceeds $u$, some coordinate has magnitude above $u/\sqrt{|\mathcal Y|}$. Thus we can
bound the probability in the integral by a union bound over coordinates and then apply~\eqref{eq:coordinate-tail}:
\begin{equation}
 \max_{t\le T}\E\|\widetilde\theta_t^\beta-\theta_t^\beta\|_2^3
 \le6|\mathcal Y|\int_0^\infty u^2
 \exp\bigg(-\frac{Bu^2}
 {10\tau|\mathcal Y|L^2e^{2\gamma\tau}\eta}\bigg)\,\mathrm d u=O((\eta/B)^{3/2}).
 \label{eq:aux-moments}
\end{equation}

\noindent\textbf{Quadratic term in the expansion of $J(\theta)$.}
We replace the conditional noise covariances along the auxiliary path by those along the exact path, and control the resulting error. The covariance replacement error satisfies
\begin{align}
 &\E\bigg\|
 \mathbf1_{\{t<\zeta^\beta\}}
 \bigg[\operatorname{diag}\bigg(
 \frac{\mathbf r}{\pi_{\widetilde\theta_t^\beta}(\cdot\mid x)}
 \bigg)-\mathbf r\mathbf r^\top\bigg]
 -\bigg[\operatorname{diag}\bigg(
 \frac{\mathbf r}{\pi_{\theta_t^\beta}(\cdot\mid x)}
 \bigg)-\mathbf r\mathbf r^\top\bigg]
 \bigg\|_{\mathrm{op}}\notag\\
 &\quad\le2|\mathcal Y|e^{2M}
 \E\|\widetilde\theta_t^\beta-\theta_t^\beta\|_\infty
 +|\mathcal Y|(e^{2M}+1)\Prob(\zeta^\beta\le t)
 \label{eq:covariance-replacement-error}\\
 &\quad\le2|\mathcal Y|e^{2M}
 \sqrt{5\tau e^{2\gamma\tau}|\mathcal Y|L^2\eta/B}
 +|\mathcal Y|(e^{2M}+1)\Prob(\zeta^\beta\le T)\notag\\
 &\quad=O\big(\sqrt{\eta/B}+\Prob(\zeta^\beta\le T)\big).
 \label{eq:covariance-replacement-bound}
\end{align}
For the first inequality, on $\{t<\zeta^\beta\}$ the two $\mathbf r\mathbf r^\top$ terms cancel, and the operator norm is bounded by $\max_y|
\pi_{\widetilde\theta_t^\beta}(y\mid x)^{-1}
-\pi_{\theta_t^\beta}(y\mid x)^{-1}
|$. Both parameter vectors have infinity norm at most $M$. Applying the mean value theorem to the identity $\pi_\theta(y\mid x)^{-1}=\sum_{y'}\exp({\theta(x,y')-\theta(x,y)})$ gives the upper bound.
On exit, only the exact-path covariance remains, with norm bounded by $
\max_y1/\pi_{\theta_t^\beta}(y\mid x)+\|\mathbf r\|_2^2$.
The second inequality uses the Cauchy--Schwarz inequality and~\eqref{eq:aux-second}.

We then evaluate the quadratic term in the expansion of $J(\theta)$. Since $\widetilde\varepsilon_t^\beta$ has conditional mean zero, cross terms between distinct updates vanish in expectation.
Also, $\nabla^2J(\theta_T^\beta)\mathbf1=0$,
so the noise projection in~\eqref{eq:innovation} does not change the
quadratic form. Consequently,
\begin{align}
 &\E\big[
 (\widetilde\theta_T^\beta-\theta_T^\beta)^\top
 \nabla^2J(\theta_T^\beta)
 (\widetilde\theta_T^\beta-\theta_T^\beta)\big]\notag\\
 &\quad=\frac1B\sum_{t=0}^{T-1}\Psi_\beta(T,t)^2
 \E\operatorname{tr}\bigg[
 \nabla^2J(\theta_T^\beta)\mathbf1_{\{t<\zeta^\beta\}}
 \bigg[\operatorname{diag}\bigg(
 \frac{\mathbf r}{\pi_{\widetilde\theta_{t}^\beta}(\cdot\mid x)}
 \bigg)-\mathbf r\mathbf r^\top\bigg]\bigg]\notag\\
 &\quad=\frac\eta B\operatorname{tr}\big[
 \nabla^2J(\theta_T^\beta)\Sigma_T^\beta\big]
 +O\bigg(\frac\eta B
 \big(\sqrt{\eta/B}+\Prob(\zeta^\beta\le T)\big)\bigg)\notag\\
 &\quad=\frac\eta B\operatorname{tr}\big[
 \nabla^2J(\theta_T^\beta)\Sigma_T^\beta\big]
 +O((\eta/B)^{3/2}).
 \label{eq:uniform-aux-covariance}
\end{align}
The second equality replaces the covariance matrices by using the definition of $\Sigma_T^\beta$ in~\eqref{eq:discrete-covariance}, the replacement error in~\eqref{eq:covariance-replacement-bound}, and
$\sum_{t=0}^{T-1}\Psi_\beta(T,t)^2\le5\tau e^{2\gamma\tau}\eta$
from the third inequality in~\eqref{eq:response-bounds}.
The expected absolute error in each trace is bounded by $|\mathcal Y|\cdot\|\nabla^2J(\theta_T^\beta)\|_{\mathrm{op}}$ times the expected operator norm of the covariance difference in~\eqref{eq:covariance-replacement-bound}.
Also, $\|\nabla^2J(\theta_T^\beta)\|_{\mathrm{op}}$ is bounded independently of $\eta$ and $\beta$. The last equality uses \eqref{eq:original-path-tail} to absorb the probability.
According to~\eqref{eq:discrete-covariance}, the third inequality in~\eqref{eq:response-bounds}, and the covariance bound above, $\sup_{\beta\in[0,1)}\|\Sigma_T^\beta\|_{\mathrm{op}}$ is bounded by $5\tau e^{2\gamma\tau}|\mathcal Y|(e^{2M}+1)$ uniformly in $\eta$. Together with the Hessian bound, this implies that the quadratic correction is $O(\eta/B)$ uniformly in $\beta$.

\vspace{5pt}
\noindent\textbf{Objective expansion and transfer.}
We combine the moment and covariance bounds in a Taylor expansion,
then transfer the resulting expected objective to the sampled process.
All derivatives of $J(\theta)$ through order three are bounded on the
region containing the exact and auxiliary paths.
Taylor's formula and the third-moment bound in~\eqref{eq:aux-moments}
give an expected remainder bounded by
\[
 \frac16\sup_{\|\vartheta\|_\infty\le M+L\tau e^{\gamma\tau}}
 \|\nabla^3J(\vartheta)\|_{\mathrm{op}}
 \E\|\widetilde\theta_T^\beta-\theta_T^\beta\|_2^3
 =O((\eta/B)^{3/2}).
\]
The bounded region includes every segment joining the exact and
auxiliary endpoints. Using $\E[\widetilde\theta_T^\beta]=\theta_T^\beta$ for the linear term,
\eqref{eq:uniform-aux-covariance} for the quadratic term, and
\eqref{eq:aux-moments} for the remainder, we obtain
\[
 \E[J(\widetilde\theta_T^\beta)]
 =J(\theta_T^\beta)+\frac\eta{2B}
 \operatorname{tr}[\nabla^2J(\theta_T^\beta)\Sigma_T^\beta]
 +O((\eta/B)^{3/2}).
\]
The expected objectives of the sampled and auxiliary processes differ
by at most $2\sup_\theta|J(\theta)|\,\Prob(\zeta^\beta\le T)
=o((\eta/B)^{3/2})$.
Here $-\gamma\log(1/\min_y\pi_{\mathrm{ref}}(y\mid x))
\le J(\theta)\le1$ for every $\theta$.
Combining this estimate with the expansion for the auxiliary process yields
\begin{equation}
 \E[J(\widehat\theta_T^\beta)]
 =J(\theta_T^\beta)+\frac\eta{2B}
 \operatorname{tr}[\nabla^2J(\theta_T^\beta)\Sigma_T^\beta]
 +O((\eta/B)^{3/2}).
 \label{eq:locality-expansion}
\end{equation}
Subtracting the expansions for $\beta$ and $0$ proves the lemma.
\end{proof}
\subsection{Proof of Lemma~\ref{lem:response}}
\label{app:response-bounds}
\begin{proof}
We bound the residual $f_t(\beta)$ and then establish the three noise-response bounds.

\vspace{5pt}
\noindent\textbf{Bound of $f_t(\beta)$.}
By the defining recurrence of $f_t(\beta)$ in Section~\ref{sec:exact}, the triangle inequality gives
\begin{align*}
 |f_{t+1}(\beta)|
 \le |f_t(\beta)|+\gamma\eta\sum_{t'=0}^t
       \omega_\beta(t,t')|f_{t'}(\beta)|\le(1+\gamma\eta)\max_{t'\le t}|f_{t'}(\beta)|,
\end{align*}
where we use $\omega_\beta(t,t')\ge0$ and
$\sum_{t'=0}^t\omega_\beta(t,t')=1$.
The same upper bound also covers all earlier $|f_{t'}(\beta)|$, so
$\max_{t'\le t+1}|f_{t'}(\beta)|\le(1+\gamma\eta)\max_{t'\le t}|f_{t'}(\beta)|.$
With $f_0(\beta)=1$, we have
$\max_{t'\le t}|f_{t'}(\beta)|
 \le(1+\gamma\eta)^t
 \le e^{\gamma\eta t}\le e^{\gamma\tau},$ which proves the bound $|f_t(\beta)|\le e^{\gamma\tau}$ in Lemma~\ref{lem:response}.

\vspace{5pt}
\noindent\textbf{Proof of the first inequality in~\eqref{eq:response-bounds}.}
Fixing $t'$ and applying the triangle
inequality to~\eqref{eq:response} yields $
 |\Psi_\beta(t+1,t')|
 \le(1+\gamma\eta)\max_{t'\le k\le t}|\Psi_\beta(k,t')|
 +\eta\omega_\beta(t,t').$
We use $\omega_\beta(t,k)\ge0$ and
$\sum_{k=t'}^{t}\omega_\beta(t,k)\le1$.
Similar to the previous argument, this also bounds $\max_{t'\le k\le t+1}|\Psi_\beta(k,t')|$.
Starting from $\Psi_\beta(t',t')=0$ and iterating gives
\begin{align}
 |\Psi_\beta(t,t')|
 &\le\eta\sum_{k=t'}^{t-1}(1+\gamma\eta)^{t-1-k}\omega_\beta(k,t')
 \le e^{\gamma\tau}\eta\sum_{k=t'}^{t-1}\omega_\beta(k,t')\\
 &= e^{\gamma\tau}\eta\sum_{k=t'}^{t-1}\bigg((1-\beta)\beta^{k-t'}
   +\frac{\beta^{2k+1-t'}}{\sum_{\ell=0}^k\beta^\ell}\bigg)\\
 &\le e^{\gamma\tau}\eta\bigg(1-\beta^{t-t'}
   +\sum_{k=t'}^{t-1}\frac{\beta^{k+1-t'}}{k+1}\bigg)
   \le e^{\gamma\tau}\eta\bigg(1+\sum_{k=t'+1}^t\frac1k \bigg).
 \label{eq:response-envelope-iteration}
\end{align}
In the first line, we use $(1+\gamma\eta)^{t-1-k}\le e^{\gamma\tau}$. The second line is given by the definition of $\omega_\beta(k,t')$. In the last line, we use $\sum_{\ell=0}^k\beta^\ell\ge(k+1)\beta^k$.
The same bound holds for $\beta=0$.
Combining these inequalities proves the first inequality in~\eqref{eq:response-bounds}.

\vspace{5pt}
\noindent\textbf{Proof of the second inequality in~\eqref{eq:response-bounds}.} Using~\eqref{eq:response-envelope-iteration} and
$\sum_{t'=0}^k\omega_\beta(k,t')=1$, we sum over $t'$ and interchange
the sums to obtain
\[
 \sum_{t'=0}^{t-1}|\Psi_\beta(t,t')|
 \le e^{\gamma\tau}\eta\sum_{t'=0}^{t-1}\sum_{k=t'}^{t-1}\omega_\beta(k,t')
 =e^{\gamma\tau}\eta\sum_{k=0}^{t-1}\sum_{t'=0}^k\omega_\beta(k,t')
 =e^{\gamma\tau}t\eta\le e^{\gamma\tau}\tau.
\]

\noindent\textbf{Proof of the third inequality in~\eqref{eq:response-bounds}.}
Expanding gives
$\sum_{t'=0}^{t-1}\left(1+\sum_{k=t'+1}^t1/k\right)^2=5t-\sum_{k=1}^t 1/k\le5t.$ Combining this identity with the first inequality in~\eqref{eq:response-bounds} proves the third inequality in~\eqref{eq:response-bounds} directly.
\end{proof}
\section{Proof of Proposition~\ref{prop:variance-main}}
\label{app:variance}
\begin{proof}
Apply the expansion proved in Appendix~\ref{app:localization} to
$J(\theta)^2$. Since
$\nabla^2(J(\theta)^2)=2J(\theta)\nabla^2J(\theta)+2\nabla J(\theta)\nabla J(\theta)^\top$, it gives
\begin{align*}
 \E[J(\widehat\theta_T^\beta)^2]
 &=J(\theta_T^\beta)^2
 +\frac\eta B J(\theta_T^\beta)
 \operatorname{tr}[\nabla^2J(\theta_T^\beta)\Sigma_T^\beta] +\frac\eta B\nabla J(\theta_T^\beta)^\top
 \Sigma_T^\beta\nabla J(\theta_T^\beta)+O((\eta/B)^{3/2}).
\end{align*}
Subtracting the square of~\eqref{eq:locality-expansion} cancels the
Hessian term and yields, uniformly in $\beta$,
\begin{equation}
 \operatorname{Var}[J(\widehat\theta_T^\beta)]
 =\frac\eta B\nabla J(\theta_T^\beta)^\top
 \Sigma_T^\beta\nabla J(\theta_T^\beta)+O((\eta/B)^{3/2}).
 \label{eq:uniform-individual-variance}
\end{equation}
\noindent\textbf{Independent comparison.}
The two variances add by independence, proving the expansion in
Proposition~\ref{prop:variance-main}. Its two leading quadratic forms
are nonnegative and uniformly bounded. Under the stated
nondegeneracy condition, the no-EMA term is bounded away from zero
as $\eta\to0$. Thus, $\operatorname{Var}\big[J(\widehat\theta_T^\beta)-J(\widehat\theta_T^0)\big]
 =\Theta(\eta/B)$.
\end{proof}
\section{Proof of Theorem~\ref{cor:fixed-beta}}
\label{app:fixed-beta}
\begin{proof}
For fixed $\beta$, we first expand the exact objective gain and
compare the two sampling corrections in
Lemma~\ref{lem:locality-main} to obtain the expected objective gain.
We then combine asymptotic normality with the vanishing
mean-to-standard-deviation ratio to derive the limiting winning
probability.

\vspace{5pt}
\noindent\textbf{Exact objective gain.}
We first state the required approximation for $f_t(\beta)$.
\begin{lemma}[Approximation of $f_t(\beta)$]
\label{lem:fixed-approximations}
For fixed $\beta\in[0,1)$ and $\tau>0$, the following approximation
for $f_t$ defined in Section~\ref{sec:exact} holds uniformly for
integers $0\le t\le T$ as $\eta=\tau/T\to0$,
\begin{equation}
 f_t(\beta)=
 \left(1-\gamma\eta-\gamma^2\eta^2\frac\beta{1-\beta}\right)^t
 +O(\eta^2).
 \label{eq:fixed-approximations}
\end{equation}
\end{lemma}
Together with $f_T(0)=(1-\gamma\eta)^T$,
Lemma~\ref{lem:fixed-approximations} gives
\begin{align}
 f_T(0)-f_T(\beta)
 &=(1-\gamma\eta)^T
 -\left(1-\gamma\eta-\gamma^2\eta^2\frac\beta{1-\beta}\right)^T
 +O(\eta^2)\notag\\
 &=T\gamma^2\eta^2\frac\beta{1-\beta}(1-\gamma\eta)^{T-1}
 +O(T^2\eta^4+\eta^2)\notag\\
 &=\gamma^2\tau e^{-\gamma\tau}\frac{\eta\beta}{1-\beta}
 +O(\eta^2).
 \label{eq:fixed-residual-difference}
\end{align}
The second line is Taylor's formula for the power.
The last line uses $T\eta=\tau$ and
$(1-\gamma\eta)^{T-1}=e^{-\gamma\tau}+O(\eta)$.
Lemma~\ref{lem:fixed-approximations} and $T\eta=\tau$ give
$f_T(0)=e^{-\gamma\tau}+O(\eta)$ and
$f_T(\beta)=e^{-\gamma\tau}+O(\eta)$.
To obtain the objective gain along the exact path, we integrate~\eqref{eq:objective-residual-derivative} from $a=f_T(\beta)$ to $a=f_T(0)$, which gives
\begin{align}
 J(\theta_T^\beta)-J(\theta_T^0)
 &=\gamma\int_{f_T(\beta)}^{f_T(0)}
 a\operatorname{Var}_{y\sim\pi_{\theta^\star+a(\theta_0-\theta^\star)}(\cdot\mid x)}
 [\theta_0(x,y)-\theta^\star(x,y)]\,\mathrm d a\notag\\
 &=\gamma e^{-\gamma\tau}
 \operatorname{Var}_{y\sim\pi_{\theta^\star+e^{-\gamma\tau}(\theta_0-\theta^\star)}}
 [\theta_0(x,y)-\theta^\star(x,y)]
 [f_T(0)-f_T(\beta)]+O(\eta^2)\notag\\
 &=C(\tau,\gamma)\frac{\eta\beta}{1-\beta}+O(\eta^2).
 \label{eq:fixed-beta-exact-app}
\end{align}
The integrand varies by $O(\eta)$ over an interval of length
$O(\eta)$, which explains the second-line remainder. The
constant in Theorem~\ref{cor:fixed-beta} is given explicitly by
\[
 C(\tau,\gamma)=\gamma^3\tau e^{-2\gamma\tau}
 \operatorname{Var}_{y\sim\pi_{\theta^\star+e^{-\gamma\tau}(\theta_0-\theta^\star)}(\cdot\mid x)}
 [\theta_0(x,y)-\theta^\star(x,y)].
\]
This expression specifies its dependence on the problem, $\tau$,
and $\gamma$. It is independent of $\eta$, $\beta$, and $B$.
The coefficient is strictly positive under the nondegeneracy condition, which implies $\theta_0\ne\theta^\star$.

\vspace{5pt}
\noindent\textbf{Difference between the quadratic terms.}
Using $f_t(0)=(1-\gamma\eta)^t$,
Lemma~\ref{lem:fixed-approximations} and the mean value theorem give
$|f_t(\beta)-f_t(0)|
\le t\gamma^2\eta^2\beta/(1-\beta)+O(\eta^2)=O(\eta)$
uniformly for $0\le t\le T$, since $t\eta\le\tau$.
Lemma~\ref{lem:exact-dynamics} then gives
$\|\theta_t^\beta-\theta_t^0\|_\infty
=|f_t(\beta)-f_t(0)|\cdot\|\theta_0-\theta^\star\|_\infty$,
and hence
$\max_{t\le T}\|\theta_t^\beta-\theta_t^0\|_\infty=O(\eta)$.
Since $J(\theta)$ is smooth and the entries of $\theta_t^\beta$ and $\theta_t^0$ are bounded uniformly in $t\le T$ and $\eta$, $\|\nabla^3J(\theta)\|_{\mathrm{op}}$ is bounded independently of $\eta$ whenever the entries of $\theta$ satisfy the same bound. This gives a Lipschitz bound for $\nabla^2J(\theta)$ with a constant independent of $\eta$. Combining this with the preceding $O(\eta)$ bound on the parameter difference gives $\|\nabla^2J(\theta_T^\beta)-\nabla^2J(\theta_T^0)\|_{\mathrm{op}}=O(\eta)$.
To compare the two sampling corrections in Lemma~\ref{lem:locality-main},
we also need the following bound on $\Sigma_T^\beta-\Sigma_T^0$.
\begin{lemma}[Covariance comparison]
\label{lem:fixed-covariance}
Let $\Sigma_T^\beta$ be defined by~\eqref{eq:discrete-covariance}.
For fixed $\beta\in[0,1)$ and $\tau>0$, as $\eta=\tau/T\to0$, we have
$\|\Sigma_T^\beta-\Sigma_T^0\|_{\mathrm{op}}=O(\eta).$
\end{lemma}
Combining this bound with the preceding Hessian estimate, we obtain
\begin{align*}
 &\left|\operatorname{tr}[\nabla^2J(\theta_T^\beta)\Sigma_T^\beta
 -\nabla^2J(\theta_T^0)\Sigma_T^0]\right|\\
 &\quad\le |\mathcal Y|\cdot
 \|\nabla^2J(\theta_T^\beta)-\nabla^2J(\theta_T^0)\|_{\mathrm{op}}\cdot
 \|\Sigma_T^\beta\|_{\mathrm{op}}
 +|\mathcal Y|\cdot\|\nabla^2J(\theta_T^0)\|_{\mathrm{op}}\cdot
 \|\Sigma_T^\beta-\Sigma_T^0\|_{\mathrm{op}}=O(\eta).
\end{align*}
Here we add and subtract $\nabla^2J(\theta_T^0)\Sigma_T^\beta$ inside the trace, then bound each trace in absolute value. The bounds in Appendix~\ref{app:localization} give $\|\Sigma_T^\beta\|_{\mathrm{op}}=O(1)$ and $\|\nabla^2J(\theta_T^0)\|_{\mathrm{op}}=O(1)$.
Lemma~\ref{lem:locality-main} and~\eqref{eq:fixed-beta-exact-app}
therefore yield
\begin{align*}
 \E[J(\widehat\theta_T^\beta)-J(\widehat\theta_T^0)]
 &=C(\tau,\gamma)\frac{\eta\beta}{1-\beta}
 +O(\eta^2)+O(\eta^2/B)+O((\eta/B)^{3/2})\\
 &=C(\tau,\gamma)\frac{\eta\beta}{1-\beta}+O(\eta^{3/2}).
\end{align*}
\noindent\textbf{Probability of improvement.}
The mean-to-standard-deviation ratio alone does not determine the
winning probability. We also need the following normal limit to characterize the distribution of the objective difference, which is
proved in the next subsection.
\begin{lemma}[Asymptotic normality]
\label{lem:fixed-normal}
Under the assumptions of Theorem~\ref{cor:fixed-beta}, the distribution of the standardized objective difference converges to a standard normal distribution as $\eta\to0$,
\begin{equation}
 \frac{J(\widehat\theta_T^\beta)-J(\widehat\theta_T^0)
 -\E[J(\widehat\theta_T^\beta)-J(\widehat\theta_T^0)]}
 {\operatorname{Std}[J(\widehat\theta_T^\beta)-J(\widehat\theta_T^0)]}
 \xrightarrow{\mathrm d}\mathcal N(0,1).
 \label{eq:fixed-standardized-clt}
\end{equation}
\end{lemma}
The mean expansion and Proposition~\ref{prop:variance-main} give
$\E[J(\widehat\theta_T^\beta)-J(\widehat\theta_T^0)]/
\operatorname{Std}[J(\widehat\theta_T^\beta)-J(\widehat\theta_T^0)]
\to0$.
Combining this estimate with the normal limit in
Lemma~\ref{lem:fixed-normal} gives
$\Prob(J(\widehat\theta_T^\beta)>J(\widehat\theta_T^0))\to1/2$.
\end{proof}
\subsection{Proofs of the auxiliary lemmas}
\begin{proof}[Proof of Lemma~\ref{lem:fixed-approximations}]
The definition of $f_t(\beta)$ in Section~\ref{sec:exact} gives $f_{t+1}(\beta)=(1-\gamma\eta)f_t(\beta)
 -\gamma\eta(\sum_{t'=0}^t\omega_\beta(t,t')f_{t'}(\beta)-f_t(\beta)).$
We express the weighted sum by $f_t(\beta)$ to obtain the desired approximation.
The bound $|f_t(\beta)|\le e^{\gamma\tau}$ in Lemma~\ref{lem:response}
and the definition of $f_t(\beta)$ give
$|f_{t'}(\beta)-f_t(\beta)|\le\gamma\eta(t-t')e^{\gamma\tau}$
for all $0\le t'\le t\le T$.
For $0\le k\le t$, we therefore have
\begin{align*}
 &|f_{t-k}(\beta)-f_t(\beta)-\gamma\eta kf_t(\beta)|\\
 &\quad=\gamma\eta\Big|\sum_{\ell=t-k}^{t-1}\sum_{t'=0}^{\ell}
 \omega_\beta(\ell,t')\big(f_{t'}(\beta)-f_t(\beta)\big)\Big|
 \le\gamma^2\eta^2e^{\gamma\tau}\sum_{\ell=t-k}^{t-1}\sum_{t'=0}^{\ell}
 \omega_\beta(\ell,t')(t-t')\\
 &\quad\le\gamma^2\eta^2e^{\gamma\tau}\sum_{\ell=t-k}^{t-1}
 (t-\ell+\beta/(1-\beta))
 =\gamma^2\eta^2e^{\gamma\tau}(k(k+1)/2+k\beta/(1-\beta)).
\end{align*}
The first equality follows by applying the definition of $f_t(\beta)$ to express $f_{t-k}(\beta)-f_{t}(\beta)$ and using $\sum_{t'=0}^{\ell}\omega_\beta(\ell,t')=1$.
For the second inequality, we write $t-t'=(t-\ell)+(\ell-t')$ and bound the weighted average of $\ell-t'$ using $\sum_{t'=0}^{\ell}(\ell-t')\omega_\beta(\ell,t')\le\beta/(1-\beta)$.
Multiplying the resulting error bound by $\omega_\beta(t,t-k)$ and summing over $k\le t$ gives
\begin{align*}
 &\Big|\sum_{k=0}^t\omega_\beta(t,t-k)f_{t-k}(\beta)-f_t(\beta)
 -\frac{\gamma\eta\beta f_t(\beta)}{1-\beta} \Big|\\
 &\quad\le\gamma^2\eta^2e^{\gamma\tau}\sum_{k=0}^t\omega_\beta(t,t-k)
 \Big(\frac{k(k+1)}{2}+\frac{k\beta}{1-\beta}\Big)+\gamma\eta|f_t(\beta)|\cdot\Big|\sum_{k=0}^t k\omega_\beta(t,t-k)-\frac{\beta}{1-\beta}\Big|\\
 &\quad\le\gamma^2\eta^2e^{\gamma\tau}\beta(1+\beta)/(1-\beta)^2
 +\gamma\eta e^{\gamma\tau}(t+1)\beta^{t+1}/(1-\beta).
\end{align*}
For the second inequality, the definition of $\omega_\beta$ gives
$0\le\beta/(1-\beta)-\sum_{k=0}^t k\omega_\beta(t,t-k)
\le(t+1)\beta^{t+1}/(1-\beta)$ and
$\sum_{k=0}^t k^2\omega_\beta(t,t-k)\le \sum_{k=0}^{\infty}k^2\beta^k/(\sum_{k=0}^{\infty}\beta^k) = \beta(1+\beta)/(1-\beta)^2$.
We also use $|f_t(\beta)|\le e^{\gamma\tau}$ in Lemma~\ref{lem:response}.
Substituting this bound into the definition of $f_{t+1}(\beta)$ yields, for $t<T$,
\begin{equation}
 \Big|f_{t+1}(\beta)-\Big(1-\gamma\eta-\frac{\gamma^2\eta^2\beta}{1-\beta}\Big)f_t(\beta)\Big|\le\frac{\gamma^3\eta^3e^{\gamma\tau}\beta(1+\beta)}{(1-\beta)^2}+\frac{\gamma^2\eta^2e^{\gamma\tau}(t+1)\beta^{t+1}}{1-\beta}.
 \label{eq:fixed-one-step-error}
\end{equation}
Repeatedly applying~\eqref{eq:fixed-one-step-error} from $f_0(\beta)=1$ gives
\begin{align*}
 \max_{t\le T}\Big|f_t(\beta)-\Big[1-\gamma\eta-\frac{\gamma^2\eta^2\beta}{1-\beta}\Big]^t\Big|&\le\frac{T\gamma^3\eta^3e^{\gamma\tau}\beta(1+\beta)}{(1-\beta)^2}
 +\frac{\gamma^2\eta^2e^{\gamma\tau}}{1-\beta}\sum_{t'=0}^{T-1}(t'+1)\beta^{t'+1}\\
 &\le\frac{\gamma^2\eta^2\beta e^{\gamma\tau}[1+\gamma\tau(1-\beta^2)]}{(1-\beta)^3}=O(\eta^2).
\end{align*}
For the first inequality, the coefficient $1-\gamma\eta-\gamma^2\eta^2\beta/(1-\beta)$ lies in $(0,1)$ when $\eta$ is sufficiently small, so the total error is bounded by the sum of the one-step error bounds.
The second inequality uses $T\eta=\tau$ and $\sum_{t'=0}^{T-1}(t'+1)\beta^{t'+1}\le\beta/(1-\beta)^2$.
This proves Lemma~\ref{lem:fixed-approximations}.
\end{proof}
\begin{proof}[Proof of Lemma~\ref{lem:fixed-covariance}]
Using~\eqref{eq:discrete-covariance}, we add and subtract the sum with coefficients $\Psi_0(T,t)^2/\eta$ and covariance matrices evaluated at $\theta_{t}^\beta$. Decomposition gives
\begin{align}
 \|\Sigma_T^\beta-\Sigma_T^0\|_{\mathrm{op}}
 &\le\frac1\eta\sum_{t=0}^{T-1}
 \big|\Psi_\beta(T,t)^2-\Psi_0(T,t)^2\big|\cdot
 \bigg\|\operatorname{diag}\bigg(
 \frac{\mathbf r}{\pi_{\theta_{t}^\beta}(\cdot\mid x)}\bigg)
 -\mathbf r\mathbf r^\top\bigg\|_{\mathrm{op}}\\
 &\quad+\frac1\eta\sum_{t=0}^{T-1}\Psi_0(T,t)^2
 \bigg\|\operatorname{diag}\bigg(
 \frac{\mathbf r}{\pi_{\theta_{t}^\beta}(\cdot\mid x)}
 -\frac{\mathbf r}{\pi_{\theta_{t}^0}(\cdot\mid x)}\bigg)
 \bigg\|_{\mathrm{op}}\\
 &\le\frac{|\mathcal Y|(e^{2M}+1)}\eta
 \max_{t<T}\big(|\Psi_\beta(T,t)|+|\Psi_0(T,t)|\big)
 \sum_{t=0}^{T-1}|\Psi_\beta(T,t)-\Psi_0(T,t)|\\
 &\quad+2\tau|\mathcal Y|e^{2M}
 \max_{t\le T}\|\theta_{t}^\beta-\theta_{t}^0\|_\infty.
\label{eq:fixed-covariance-decomposition}
\end{align}
For the second inequality, we bound the two matrix operator norms as in~\eqref{eq:covariance-replacement-bound} by $|\mathcal Y|(e^{2M}+1)$ and $2|\mathcal Y|e^{2M}\|\theta_{t}^\beta-\theta_{t}^0\|_\infty$, respectively, with $M$ defined in Appendix~\ref{app:localization}. We also use $|\Psi_\beta(T,t)^2-\Psi_0(T,t)^2| =|\Psi_\beta(T,t)-\Psi_0(T,t)|\cdot|\Psi_\beta(T,t)+\Psi_0(T,t)|$ to bound the first sum. For the remaining term, we use $\sum_{t=0}^{T-1}\Psi_0(T,t)^2/\eta\le T\eta=\tau$, since~\eqref{eq:response} gives $\Psi_0(T,t)=\eta(1-\gamma\eta)^{T-t-1}\in[0,\eta]$ for $\gamma\eta\le1$. We next bound the three quantities. First, applying~\eqref{eq:response-envelope-iteration} to both responses and taking the maximum over $t'=0,\ldots,t-1$ gives, for all $1\le t\le T$,
\begin{equation}
 \max_{t'<t}\big(|\Psi_\beta(t,t')|+|\Psi_0(t,t')|\big)
 \le e^{\gamma\tau}\eta
 \Big(\max_{t'<t}\sum_{k=t'}^{t-1}\omega_\beta(k,t')+1\Big)
 \le e^{\gamma\tau}\eta\Big(\frac1{1-\beta}+1\Big)
 =O(\eta).
 \label{eq:fixed-response-bound}
\end{equation}
Here we use $\sum_{k=t'}^{t-1}\omega_0(k,t')=1$ and $\sum_{k=t'}^{t-1}\omega_\beta(k,t')\leq\sum_{k=t'}^{t-1}\beta^{k-t'}\le1/(1-\beta)$. For the sum of response differences, we sum $\Psi_\beta(\ell+1,t')-\Psi_\beta(\ell,t')$ over $\ell=t',\ldots,t-1$ using~\eqref{eq:response} and interchange the sums to obtain
\begin{align*}
 \Psi_\beta(t,t')
 &=\eta\sum_{\ell=t'}^{t-1}\omega_\beta(\ell,t')
 -\gamma\eta\sum_{\ell=t'}^{t-1}\sum_{k=t'}^{\ell}
 \omega_\beta(\ell,k)\Psi_\beta(k,t')\\
 &=\eta\sum_{\ell=t'}^{t-1}\omega_\beta(\ell,t')
 -\gamma\eta\sum_{k=t'+1}^{t-1}
 \left(\sum_{\ell=k}^{t-1}\omega_\beta(\ell,k)\right)\Psi_\beta(k,t').
\end{align*}
The second equality omits the $k=t'$ term since $\Psi_\beta(t',t')=0$. We subtract the identity at $\beta=0$, where each cumulative weight equals one. Summing over $t'=0,\ldots,t-1$, we have
\begin{align}
 \sum_{t'=0}^{t-1}|\Psi_\beta(t,t')-\Psi_0(t,t')|
 &\le\eta\sum_{t'=0}^{t-1}\big|\sum_{\ell=t'}^{t-1}\omega_\beta(\ell,t')-1\big|
 +\gamma\eta\sum_{t'=1}^{t-1}\sum_{k=0}^{t'-1}|\Psi_\beta(t',k)-\Psi_0(t',k)|\\
 &\quad+\gamma\eta\sum_{t'=1}^{t-1}\Big[
 \big|\sum_{\ell=t'}^{t-1}\omega_\beta(\ell,t')-1\big|
 \sum_{k=0}^{t'-1}|\Psi_\beta(t',k)|\Big]\\
 &\le 2\eta\beta(1+\gamma\tau e^{\gamma\tau})/(1-\beta)
 +\gamma\eta\sum_{t'=1}^{t-1}\sum_{k=0}^{t'-1}|\Psi_\beta(t',k)-\Psi_0(t',k)|\\
 &\le 2\eta\beta(1+\gamma\tau e^{\gamma\tau})e^{\gamma\tau}/(1-\beta)=O(\eta).
\label{eq:fixed-response-difference-sum}
\end{align}
For the second inequality, we have $\sum_{t'=0}^{t-1}\sum_{\ell=t'}^{t-1}\omega_\beta(\ell,t') =\sum_{\ell=0}^{t-1}\sum_{t'=0}^{\ell}\omega_\beta(\ell,t')=t$. Thus $\sum_{t'=0}^{t-1}|\sum_{\ell=t'}^{t-1}\omega_\beta(\ell,t')-1| =2\sum_{t'=0}^{t-1}\max\{1-\sum_{\ell=t'}^{t-1}\omega_\beta(\ell,t'),0\}$. Since $\omega_\beta(\ell,t')\ge(1-\beta)\beta^{\ell-t'}$, we obtain $\sum_{t'=0}^{t-1}|\sum_{\ell=t'}^{t-1}\omega_\beta(\ell,t')-1| \le2\sum_{t'=0}^{t-1}\beta^{t-t'}\le2\beta/(1-\beta)$. We also use $\sum_{k=0}^{t'-1}|\Psi_\beta(t',k)|\le\tau e^{\gamma\tau}$ from the second inequality in~\eqref{eq:response-bounds}. The last inequality follows by iterating the preceding bound and using $(1+\gamma\eta)^{t-1} \le e^{\gamma\eta(t-1)} \le e^{\gamma\tau}$. For the parameter differences, we substitute the representations from Lemma~\ref{lem:exact-dynamics} for $\theta_t^\beta$ and $\theta_t^0$ to obtain
\begin{equation}
 \max_{t\le T}\|\theta_t^\beta-\theta_t^0\|_\infty
 =\max_{t\le T}|f_t(\beta)-f_t(0)|\cdot
 \|\theta_0-\theta^\star\|_\infty=O(\eta).
\label{eq:fixed-parameter-difference}
\end{equation}
The last bound uses Lemma~\ref{lem:fixed-approximations} and the mean value theorem, which give $ \max_{t\le T}|f_t(\beta)-f_t(0)| \le T\gamma^2\eta^2\beta/(1-\beta) +O(\eta^2) =O(\eta)$. Substituting~\eqref{eq:fixed-response-bound}, \eqref{eq:fixed-response-difference-sum}, and~\eqref{eq:fixed-parameter-difference} into~\eqref{eq:fixed-covariance-decomposition} yields $\|\Sigma_T^\beta-\Sigma_T^0\|_{\mathrm{op}}=O(\eta)$.
\end{proof}
\begin{proof}[Proof of Lemma~\ref{lem:fixed-normal}]
We first analyze the centered objective difference for the auxiliary processes and then transfer the normal limit to the sampled processes.

\vspace{5pt}
\noindent\textbf{Expansion of the auxiliary objective.}
Taylor's formula at $\theta_T^\beta$ gives
$
 J(\widetilde\theta_T^\beta)=J(\theta_T^\beta)
 +\nabla J(\theta_T^\beta)^\top(\widetilde\theta_T^\beta-\theta_T^\beta)+R_T^\beta,
$
where $R_T^\beta$ is the remainder after truncating this expansion at the linear term. Since $\E[\widetilde\theta_T^\beta]=\theta_T^\beta$ by Appendix~\ref{app:localization}, the expected linear term is $\nabla J(\theta_T^\beta)^\top(\E[\widetilde\theta_T^\beta]-\theta_T^\beta)=0$.
Subtracting expectations in the Taylor expansion gives
\begin{equation}
 J(\widetilde\theta_T^\beta)-\E[J(\widetilde\theta_T^\beta)]
 =\nabla J(\theta_T^\beta)^\top(\widetilde\theta_T^\beta-\theta_T^\beta)
 +R_T^\beta-\E[R_T^\beta].
 \label{eq:normal-linear-approximation}
\end{equation}
According to~\eqref{eq:response} and~\eqref{eq:stopping}, we write the linear term as $\nabla J(\theta_T^\beta)^\top(\widetilde\theta_T^\beta-\theta_T^\beta)=\sum_{t=0}^{T-1}\sum_{b=1}^B Y_{t,b}^\beta$, where
$
 Y_{t,b}^\beta=\Psi_\beta(T,t)\mathbf1_{\{t<\zeta^\beta\}}
 \nabla J(\theta_T^\beta)^\top\varepsilon_{t,b}^\beta/B.$

\vspace{5pt}
\noindent\textbf{Normal limit for a single run.}
We next prove that
$(\sum_{t=0}^{T-1}\sum_{b=1}^B Y_{t,b}^\beta)/
\operatorname{Std}[J(\widehat\theta_T^\beta)-J(\widehat\theta_T^0)]\xrightarrow{\mathrm d}\mathcal N(0,1/2).$ Let $\mathscr H_{t,b-1}^\beta$ denote the history before response $b$ at update $t$, and reveal the $TB$ responses in increasing order of $t$ and then $b$. Since \(\E[Y_{t,b}^\beta\mid\mathscr H_{t,b-1}^\beta]=0\), these variables form a martingale difference sequence.
We apply the martingale central limit theorem stated in~\citet[Theorem~E.4.1]{douc2018markov} to the partial sums of $Y_{t,b}^\beta/\operatorname{Std}[J(\widehat\theta_T^\beta)-J(\widehat\theta_T^0)]$ as $T\to\infty$.
Using~\eqref{eq:innovation} and~\eqref{eq:discrete-covariance}, we write the covariance replacement as
\begin{align}
 &\sum_{t=0}^{T-1}\sum_{b=1}^B
 \E[(Y_{t,b}^\beta)^2\mid\mathscr H_{t,b-1}^\beta]
 -\frac\eta B\nabla J(\theta_T^\beta)^\top\Sigma_T^\beta\nabla J(\theta_T^\beta)\notag\\
 &\quad=\sum_{t=0}^{T-1}\sum_{b=1}^B
 \frac{\Psi_\beta(T,t)^2\mathbf1_{\{t<\zeta^\beta\}}}{B^2}
 \nabla J(\theta_T^\beta)^\top
 \bigg[\operatorname{diag}\bigg(
 \frac{\mathbf r}{\pi_{\widetilde\theta_t^\beta}(\cdot\mid x)}
 \bigg)-\mathbf r\mathbf r^\top\bigg]\nabla J(\theta_T^\beta)\notag\\
 &\qquad-\sum_{t=0}^{T-1}\sum_{b=1}^B
 \frac{\Psi_\beta(T,t)^2}{B^2}
 \nabla J(\theta_T^\beta)^\top
 \bigg[\operatorname{diag}\bigg(
 \frac{\mathbf r}{\pi_{\theta_t^\beta}(\cdot\mid x)}
 \bigg)-\mathbf r\mathbf r^\top\bigg]\nabla J(\theta_T^\beta)
 \label{eq:normal-conditional-variance}
\end{align}
The third inequality in~\eqref{eq:response-bounds} gives $\sum_{t=0}^{T-1}\sum_{b=1}^B\Psi_\beta(T,t)^2/B^2\le5\tau e^{2\gamma\tau}\eta/B$.
Equations~\eqref{eq:original-path-tail} and~\eqref{eq:covariance-replacement-bound} bound the expected operator norm of the difference between the auxiliary covariance and the exact-path covariance by $O(\sqrt{\eta/B})$.
Together with the boundedness of $\nabla J(\theta_T^\beta)$, these bounds give an expected absolute value of $O((\eta/B)^{3/2})$ for the right-hand side. For the term $\nabla J(\theta_T^\beta)^\top\Sigma_T^\beta\nabla J(\theta_T^\beta)$, the difference between the EMA and no-EMA versions can be bounded by
\begin{align}
&\big|
\nabla J(\theta_T^\beta)^\top\Sigma_T^\beta\nabla J(\theta_T^\beta)
-\nabla J(\theta_T^0)^\top\Sigma_T^0\nabla J(\theta_T^0)
\big|\\
&\quad\le
\|\nabla J(\theta_T^\beta)-\nabla J(\theta_T^0)\|_2\cdot
\|\Sigma_T^\beta\|_{\mathrm{op}}
\big(\|\nabla J(\theta_T^\beta)\|_2+\|\nabla J(\theta_T^0)\|_2\big)\notag\\
&\qquad+
\|\nabla J(\theta_T^0)\|_2^2\cdot
\|\Sigma_T^\beta-\Sigma_T^0\|_{\mathrm{op}}
=O(\eta).
\label{eq:ema-difference-bound}
\end{align}
Lemma~\ref{lem:fixed-covariance}, the parameter difference bound~\eqref{eq:fixed-parameter-difference}, the Lipschitz continuity of $\nabla J(\theta)$ and the boundedness of $\nabla J(\theta)$ and $\Sigma_T^\beta$ give the last bound. According to Proposition~\ref{prop:variance-main}, $\operatorname{Var}[J(\widehat\theta_T^\beta)-J(\widehat\theta_T^0)]$ is $\Theta(\eta/B)$. Dividing~\eqref{eq:normal-conditional-variance} by this variance gives, as $\eta\to 0$ ($T\to\infty$),
\begin{equation}
 \Big(\sum_{t=0}^{T-1}\sum_{b=1}^B
 \E[(Y_{t,b}^\beta)^2\mid\mathscr H_{t,b-1}^\beta]\Big)/\operatorname{Var}[J(\widehat\theta_T^\beta)-J(\widehat\theta_T^0)]
 \xrightarrow{\mathbb P} 1/2.
 \label{eq:normal-variance-limit}
\end{equation}
The expected absolute value of the right-hand side of~\eqref{eq:normal-conditional-variance} divided by this variance is $O(\sqrt{\eta/B})$. Markov's inequality therefore gives convergence in probability to zero. The difference~\eqref{eq:ema-difference-bound} gives the limit $1/2$. To apply the martingale central limit theorem, we also need to bound the conditional second moments
$\sum_{t=0}^{T-1}\sum_{b=1}^B
 \E\big[(Y_{t,b}^\beta)^2
 \mathbf1_{\{|Y_{t,b}^\beta|>\epsilon\operatorname{Std}[J(\widehat\theta_T^\beta)-J(\widehat\theta_T^0)]\}}
 \mid\mathscr H_{t,b-1}^\beta\big]$, where $\epsilon>0$ is a fixed constant.
The definition of $Y_{t,b}^\beta$ gives
\begin{align*}
 \max_{t<T,\,b\le B}|Y_{t,b}^\beta|
 &\le\frac{\|\nabla J(\theta_T^\beta)\|_2}{B}
 \max_{t<T,\,b\le B}\big(
 |\Psi_\beta(T,t)|\cdot
 \|\mathbf1_{\{t<\zeta^\beta\}}\varepsilon_{t,b}^\beta\|_2\big)\\
 &\le\frac{\sqrt{|\mathcal Y|}Le^{\gamma\tau}\eta}{B}
 \|\nabla J(\theta_T^\beta)\|_2
 \max_{t<T}\sum_{k=t}^{T-1}\omega_\beta(k,t)\\
 &\le\frac{\sqrt{|\mathcal Y|}Le^{\gamma\tau}\eta}{B(1-\beta)}
 \|\nabla J(\theta_T^\beta)\|_2=O(\eta/B).
\end{align*}
The Cauchy--Schwarz inequality gives the first line. The second inequality uses the stopped noise bound $\|\mathbf1_{\{t<\zeta^\beta\}}\varepsilon_{t,b}^\beta\|_\infty\le L$ in Appendix~\ref{app:localization} and the first line of~\eqref{eq:response-envelope-iteration}.
For the last inequality, we use $\sum_{k=t}^{T-1}\omega_\beta(k,t)\le\sum_{k=t}^{T-1}\beta^{k-t}\le1/(1-\beta)$. The boundedness of $\nabla J(\theta_T^\beta)$ in Appendix~\ref{app:localization} gives the final bound.
Since $\operatorname{Std}[J(\widehat\theta_T^\beta)-J(\widehat\theta_T^0)]=\Theta(\sqrt{\eta/B})$ according to Proposition~\ref{prop:variance-main}, the conditional second moments satisfy, as $\eta\to 0$ ($T\to\infty$),
\begin{equation}
  \Big(\sum_{t=0}^{T-1}\sum_{b=1}^B
 \E\big[(Y_{t,b}^\beta)^2
 \mathbf1_{\{|Y_{t,b}^\beta|>\epsilon\operatorname{Std}[J(\widehat\theta_T^\beta)-J(\widehat\theta_T^0)]\}}
 \mid\mathscr H_{t,b-1}^\beta\big]\Big)/\operatorname{Var}[J(\widehat\theta_T^\beta)-J(\widehat\theta_T^0)]\xrightarrow{\mathbb P}0.
\label{eq:normal-lindeberg}
\end{equation}
Here we use $\max_{t<T,\,b\le B}|Y_{t,b}^\beta|/\operatorname{Std}[J(\widehat\theta_T^\beta)-J(\widehat\theta_T^0)]=O(\sqrt{\eta/B})\to0$. Thus for each fixed $\epsilon>0$, every indicator is zero for all sufficiently small $\eta$.
Equations~\eqref{eq:normal-variance-limit} and~\eqref{eq:normal-lindeberg} verify the two conditions of the theorem with limiting variance $1/2$, giving
\begin{equation}
 \Big(\sum_{t=0}^{T-1}\sum_{b=1}^B Y_{t,b}^\beta\Big)
 /\operatorname{Std}[J(\widehat\theta_T^\beta)-J(\widehat\theta_T^0)]
 \xrightarrow{\mathrm d}\mathcal N(0,1/2).
 \label{eq:normal-linear-limit}
\end{equation}
\noindent\textbf{Distribution of $J(\widetilde\theta_T^\beta)-\E[J(\widetilde\theta_T^\beta)]$.} The remainder in the Taylor expansion~\eqref{eq:normal-linear-approximation} satisfies $|R_T^\beta|\le\sup_{\|\vartheta\|_\infty\le M+L\tau e^{\gamma\tau}}\|\nabla^2J(\vartheta)\|_{\mathrm{op}}\cdot\|\widetilde\theta_T^\beta-\theta_T^\beta\|_2^2/2$, where $M$ and $L$ are constants defined in Appendix~\ref{app:localization}. Therefore,~\eqref{eq:aux-second} gives $\E|R_T^\beta-\E[R_T^\beta]|\le2\E|R_T^\beta|=O(\eta/B)$. Dividing by $\operatorname{Std}[J(\widehat\theta_T^\beta)-J(\widehat\theta_T^0)]$ and applying Markov's inequality gives $|R_T^\beta-\E[R_T^\beta]|/\operatorname{Std}[J(\widehat\theta_T^\beta)-J(\widehat\theta_T^0)] \xrightarrow{\mathbb P}0$. Applying Slutsky's theorem~\citep[Chapter~3, Exercise~31]{kroese2023advanced} to~\eqref{eq:normal-linear-approximation}, using~\eqref{eq:normal-linear-limit} and the remainder bound above, gives
$
  (J(\widetilde\theta_T^\beta)-\E[J(\widetilde\theta_T^\beta)]
  )/\operatorname{Std}[J(\widehat\theta_T^\beta)-J(\widehat\theta_T^0)]
  \xrightarrow{\mathrm d}\mathcal N(0,1/2).$
The same argument holds for the no-EMA run. Combining the two independent limits gives
\begin{equation}
\frac{J(\widetilde\theta_T^\beta)-J(\widetilde\theta_T^0)
 -\E[J(\widetilde\theta_T^\beta)-J(\widetilde\theta_T^0)]}
 {\operatorname{Std}[J(\widehat\theta_T^\beta)-J(\widehat\theta_T^0)]}
 \xrightarrow{\mathrm d}\mathcal N(0,1).
\end{equation}
\noindent\textbf{Transfer to the sampled processes.}
As the auxiliary process and the sampled process differ only after the exit time, we have
\begin{equation}
 \E|J(\widehat\theta_T^\beta)-J(\widetilde\theta_T^\beta)|
 \le2\sup_\theta|J(\theta)|\,\Prob(\zeta^\beta\le T)=o(\sqrt{\eta/B}).
\end{equation}
Here we use the exit probability bound in~\eqref{eq:original-path-tail} and $-\gamma\log(1/\min_y\pi_{\mathrm{ref}}(y\mid x))
\le J(\theta)\le1$.
The centered discrepancy satisfies $\mathbb E|J(\widehat\theta_T^\beta)-J(\widetilde\theta_T^\beta)-\mathbb E[J(\widehat\theta_T^\beta)-J(\widetilde\theta_T^\beta)]|\le2\mathbb E|J(\widehat\theta_T^\beta)-J(\widetilde\theta_T^\beta)|=o(\sqrt{\eta/B})$. Since the standard deviation is $\Theta(\sqrt{\eta/B})$ by Proposition~\ref{prop:variance-main}, Markov's inequality gives
\[
 \frac{J(\widehat\theta_T^\beta)-J(\widetilde\theta_T^\beta)-\mathbb E[J(\widehat\theta_T^\beta)-J(\widetilde\theta_T^\beta)]}{\operatorname{Std}[J(\widehat\theta_T^\beta)-J(\widehat\theta_T^0)]}\xrightarrow{\mathbb P}0.
\]
The same estimate holds when $\beta=0$. Slutsky's theorem therefore transfers the standard normal limit to the standardized difference of the sampled objectives.
\end{proof}
\section{Vanishing expected gain as \texorpdfstring{$\eta/(1-\beta)\to0$}{eta/(1-beta) tends to zero}}
\label{app:vanishing}
We formalize and prove the vanishing expected gain claim at the beginning of Section~\ref{sec:fixed-horizon}.
\begin{lemma}[Vanishing expected gain]
\label{lem:vanishing-gain}
Fix $B\ge1$ and $\tau>0$, and consider the two runs in
Section~\ref{sec:stochastic}. For any deterministic
$\beta\in[0,1)$, possibly depending on $\eta$, if $\eta/(1-\beta)\to0$ as $\eta=\tau/T\to0$, then the expected gain satisfies
 $\E[J(\widehat\theta_T^\beta)-J(\widehat\theta_T^0)]
 \to0$ as $\eta\to 0$.
\end{lemma}
\begin{proof}
We first bound the exact gain and then control the sampling correction.

\vspace{5pt}
\noindent\textbf{Exact difference.}
Subtracting the definitions of $f_T(\beta)$ and $f_T(0)$ in Section~\ref{sec:exact} gives
\begin{align}
 |f_T(\beta)-f_T(0)|&=\big|(1-\gamma\eta)\big(f_{T-1}(\beta)-f_{T-1}(0)\big)
 -\gamma\eta\sum_{t'=0}^{T-1}\omega_\beta(T-1,t')\big(f_{t'}(\beta)-f_{T-1}(\beta)\big)\big|\notag\\
 &=\gamma\eta\big|\sum_{t=0}^{T-1}(1-\gamma\eta)^{T-t-1}
 \sum_{t'=0}^t\omega_\beta(t,t')\big(f_{t'}(\beta)-f_t(\beta)\big)\big|\notag\\
 &\le\gamma^2\eta^2e^{\gamma\tau}\sum_{t=0}^{T-1}\sum_{t'=0}^t(t-t')\omega_\beta(t,t')
 \le\gamma^2\tau e^{\gamma\tau}\eta\beta/(1-\beta).
 \label{eq:vanishing-exact-difference}
\end{align}
The first line uses $\sum_{t'=0}^{T-1}\omega_\beta(T-1,t')=1$ to add $\gamma \eta \sum_{t'=0}^{T-1}\omega_\beta(T-1,t')f_{T-1}(\beta)$ and subtract $\gamma \eta f_{T-1}(\beta)$.
Repeatedly substituting this relation and using $f_0(\beta)=f_0(0)=1$ yields the second line. The definition of $f_t(\beta)$ and the bound $|f_t(\beta)|\le e^{\gamma\tau}$ in Lemma~\ref{lem:response} give $|f_{t'}(\beta)-f_t(\beta)|\le\gamma\eta e^{\gamma\tau}(t-t')$. Applying $(1-\gamma\eta)^{T-t-1}\le 1$, $T\eta=\tau$ and $\sum_{t'=0}^t(t-t')\omega_\beta(t,t')\le\beta/(1-\beta)$, which follows by summing the weights defined in Section~\ref{sec:setting}, gives the last line. For the objective difference,
we have
\begin{align*}
 |J(\theta_T^\beta)-J(\theta_T^0)|
 \le\sup_{\|\theta\|_\infty\le M}\|\nabla J(\theta)\|_2\cdot\|\theta_T^\beta-\theta_T^0\|_2\leq \frac{\gamma^2\tau e^{\gamma\tau}\eta\beta\|\theta_0-\theta^\star\|_2\sup_{\|\theta\|_\infty\le M}\|\nabla J(\theta)\|_2}{1-\beta}.
\end{align*}
The first inequality follows from the mean value theorem. The second inequality uses Lemma~\ref{lem:exact-dynamics} to express $\|\theta_T^\beta-\theta_T^0\|_2$ in terms of $|f_T(\beta)-f_T(0)|$ and applies~\eqref{eq:vanishing-exact-difference}.
The supremum is finite and independent of $\eta$ and $\beta$, since $J(\theta)$ is smooth and the infinity norms of $\theta_T^\beta$ and $\theta_T^0$ are bounded by $M$, which is defined in Appendix~\ref{app:localization}.

\vspace{5pt}
\noindent\textbf{Expected gain.}
Lemma~\ref{lem:locality-main} and the uniform Hessian bound in Appendix~\ref{app:localization} imply that the expected and exact gains differ by $O(\eta/B)$ uniformly in $\beta$. Hence, there exists a constant $C$ independent of $\eta$ and $\beta$ such that
\[
 |\E[J(\widehat\theta_T^\beta)-J(\widehat\theta_T^0)]|
 \le C\Big(\frac{\eta\beta}{1-\beta}+\frac\eta B\Big).
\]
The upper bound goes to zero when $\eta/{(1-\beta)}$ and $\eta$ both go to zero.
\end{proof}
\section{Proof of Theorem~\ref{thm:fixed-horizon}}
\label{app:fixed-horizon}
\label{app:continuous}
\begin{proof}
We identify the limiting exact gain at fixed $\rho$, then use Lemma~\ref{lem:locality-main} and Proposition~\ref{prop:variance-main} to control the sampled comparison.

\vspace{5pt}
\noindent\textbf{Continuous residual version.}
Here $\beta=1-\eta/\rho$ varies with $\eta$, so the fixed-$\beta$ analysis in Section~\ref{sec:fixed-beta} does not apply directly. For fixed $\rho>0$, we introduce a continuous residual $f(s;\rho)$ at optimization time $s\ge0$, corresponding to $f_t(\beta)$ at $s=t\eta$. It describes the remaining displacement from $\theta^\star$ through $\theta(s;\rho)=\theta^\star+f(s;\rho)(\theta_0-\theta^\star)$ and is defined by
\begin{equation}
 \partial_s f(s;\rho)
 =-\gamma\frac{\int_0^s e^{-(s-u)/\rho}f(u;\rho)\,\mathrm d u}
 {\rho(1-e^{-s/\rho})},\qquad f(0;\rho)=1.
 \label{eq:continuous-path-definition}
\end{equation}
The normalized average at $s=0$ is interpreted by continuity. The following lemma bounds the error between $f_t(\beta)$ and its continuous approximation $f(t\eta;\rho)$. Its proof is provided in Appendix~\ref{app:continuous-path-proof}.
\begin{lemma}[Continuous approximation of the exact path]
\label{lem:continuous-path}
Let $f_t$ be defined in Section~\ref{sec:exact} and $f(s;\rho)$ by~\eqref{eq:continuous-path-definition}. For $0<\eta<\rho$, $\beta=1-\eta/\rho$, and $T\eta\le\tau$,
\begin{equation}
 \max_{t\le T}|f_t(\beta)-f(t\eta;\rho)|
 \le\frac72\gamma e^{2\gamma\tau}\eta.
 \label{eq:continuous-path-bound}
\end{equation}
\end{lemma}
At $T\eta=\tau$, Lemma~\ref{lem:continuous-path} and Lemma~\ref{lem:exact-dynamics} give $\|\theta_T^\beta-\theta(\tau;\rho)\|_2=O(\eta)$. Since $f_T(0)=(1-\gamma\eta)^T=e^{-\gamma\tau}+O(\eta)$, we have $\|\theta_T^0-[\theta^\star+e^{-\gamma\tau}(\theta_0-\theta^\star)]\|_2=O(\eta)$. Applying the mean value theorem to $J(\theta)$ for each pair of parameter vectors yields
\begin{equation}
 J(\theta_T^\beta)-J(\theta_T^0)
 =J(\theta(\tau;\rho))
 -J(\theta^\star+e^{-\gamma\tau}(\theta_0-\theta^\star))+O(\eta).
 \label{eq:fixed-horizon-exact-app}
\end{equation}
The discrete parameter vectors have infinity norm at most $M$ as defined in Appendix~\ref{app:localization}, and their continuous limits satisfy the same bound. Since $J(\theta)$ is smooth, $\sup_{\|\vartheta\|_\infty\le M}\|\nabla J(\vartheta)\|_2$ is finite and independent of $\eta$. Thus each $O(\eta)$ parameter error produces an $O(\eta)$ objective error.

\vspace{5pt}
\noindent\textbf{Positive gain in the safe range.}
For $0<\rho\le1/(4\gamma)$ and $\eta<\rho$, we have $ 0<\beta=1-\eta/\rho\le1-4\gamma\eta \le((1-\gamma\eta)/(1+\gamma\eta))^2$, which lies in the safe range in Theorem~\ref{thm:exact-benefit}. Multiplying the ratio bounds $f_{t'}(\beta)/f_{t'-1}(\beta)\ge\sqrt\beta$ from the proof of Theorem~\ref{thm:exact-benefit} gives
\[
 f_t(\beta)
 =\prod_{t'=1}^t\big(f_{t'}(\beta)/f_{t'-1}(\beta)\big)
 \ge\beta^{t/2}
 =(1-\eta/\rho)^{t/2}.\]
By~\eqref{eq:continuous-path-bound}, letting $\eta\to0$ with $t\eta\to s$ yields $ f(s;\rho)\ge e^{-s/(2\rho)}>0. $ The integral in~\eqref{eq:continuous-path-definition} is positive, so $f(s;\rho)$ is strictly decreasing. Its average over preceding times is therefore strictly larger than its current value, implying $\partial_s f(s;\rho)<-\gamma f(s;\rho)$ for $s>0$. Hence we have $\frac{\mathrm d}{\mathrm ds}\big[e^{\gamma s}f(s;\rho)\big] =e^{\gamma s}\big[\partial_s f(s;\rho)+\gamma f(s;\rho)\big]<0.$ Since $f(0;\rho)=1$, this gives, for all $s>0$,
\begin{equation}
 0<f(s;\rho)<e^{-\gamma s}.
 \label{eq:continuous-safe-path}
\end{equation}
By~\eqref{eq:objective-residual-derivative}, $J(\theta^\star+a(\theta_0-\theta^\star))$ is strictly decreasing in $a>0$. Hence~\eqref{eq:continuous-safe-path} gives $J(\theta(\tau;\rho))> J(\theta^\star+e^{-\gamma\tau}(\theta_0-\theta^\star))$. The coefficient in the main theorem is therefore
\begin{align}
 C(\tau,\gamma,\rho)
 &:=\frac{J(\theta(\tau;\rho))
 -J(\theta^\star+e^{-\gamma\tau}(\theta_0-\theta^\star))}{\rho}\notag\\
 &=\frac\gamma\rho\int_{f(\tau;\rho)}^{e^{-\gamma\tau}}
 a\operatorname{Var}_{y\sim\pi_{\theta^\star+a(\theta_0-\theta^\star)}(\cdot\mid x)}
 [\theta_0(x,y)-\theta^\star(x,y)]\,\mathrm d a>0.
 \label{eq:constant-gain-coefficient}
\end{align}
This specifies its dependence on $\rho$, $\tau$, $\gamma$, and the fixed problem. It does not depend on $\eta$, $\beta$, or $B$. Using $\rho=\beta\rho+\eta$ in~\eqref{eq:fixed-horizon-exact-app} gives $J(\theta_T^\beta)-J(\theta_T^0) =C(\tau,\gamma,\rho)\beta\rho+O(\eta)$.

\vspace{5pt}
\noindent\textbf{Expected gain and winning probability.}
Lemma~\ref{lem:locality-main}, together with the uniform Hessian bound in Appendix~\ref{app:localization}, gives
\[
 \E[J(\widehat\theta_T^\beta)-J(\widehat\theta_T^0)]
 =J(\theta_T^\beta)-J(\theta_T^0)+O(\eta/B)=C(\tau,\gamma,\rho)\beta\rho+O(\eta)+O(\eta/B).
\]
The $O(\eta)$ error is the discrete-path approximation, and the $O(\eta/B)$ term is the sampling correction. At each fixed positive $\rho$, the expected gain tends to $C(\tau,\gamma,\rho)\rho>0$. Thus it is at least half this limit for sufficiently small $\eta$. Proposition~\ref{prop:variance-main} gives $ \operatorname{Var}[J(\widehat\theta_T^\beta)-J(\widehat\theta_T^0)] \le C\eta/B$ for sufficiently small $\eta$, where $C$ is a constant independent of $\eta$ and $\beta$. Chebyshev's inequality then gives
\begin{equation}
 \Prob(J(\widehat\theta_T^\beta)\le J(\widehat\theta_T^0))
 \le\frac{\operatorname{Var}[J(\widehat\theta_T^\beta)-J(\widehat\theta_T^0)]}
 {\big(\E[J(\widehat\theta_T^\beta)-J(\widehat\theta_T^0)]\big)^2}
 \le\frac{4C}{C(\tau,\gamma,\rho)^2\rho^2}\frac\eta B
 =O(\eta/B).
\end{equation}
\end{proof}
\subsection{Proof of Lemma~\ref{lem:continuous-path}}
\label{app:continuous-path-proof}
\begin{proof}
We construct the continuous solution and estimate the error of the discrete approximation.

\vspace{5pt}
\noindent\textbf{Construction of the continuous solution.}
We first establish existence and uniqueness for
\eqref{eq:continuous-path-definition} by successive substitution in
its integral form. According to the definition of the continuous version in~\eqref{eq:continuous-path-definition}, we have $f(v;\rho)=f(0;\rho)-\gamma\int_0^v((({\int_0^s e^{-(s-u)/\rho}f(u;\rho)\mathrm d u})/
 {\rho(1-e^{-s/\rho})})\mathrm d s)$. Define the iterates $f^{[n]}(v;\rho)$ for $n\ge0$ by
\[
 f^{[n+1]}(v;\rho)
 =1-\gamma\int_0^v
 \frac{\int_0^s e^{-(s-u)/\rho}f^{[n]}(u;\rho)\mathrm d u}
 {\rho(1-e^{-s/\rho})}\mathrm d s,
\]
where we start from $f^{[0]}(v;\rho)=1$ and $f(0;\rho)=1$. Since the term ${\int_0^s e^{-(s-u)/\rho}f^{[n]}(u;\rho)\mathrm d u}/(
 {\rho(1-e^{-s/\rho})})$ can be interpreted as a weighted average of $f^{[n]}(u;\rho)$ over $[0,s]$, it is bounded by the maximum absolute value on $[0,s]$. Subtracting two successive iterates therefore gives, for $n\ge1$, $ |f^{[n+1]}(v;\rho)-f^{[n]}(v;\rho)|
 \le\gamma\int_0^v\max_{u\le s}
 |f^{[n]}(u;\rho)-f^{[n-1]}(u;\rho)|\,\mathrm d s.$
Since $f^{[1]}(v;\rho)-f^{[0]}(v;\rho)=-\gamma v$, induction yields
\begin{equation}
 \max_{s\le v}
 |f^{[n+1]}(s;\rho)-f^{[n]}(s;\rho)|
 \le\frac{(\gamma v)^{n+1}}{(n+1)!}.
 \label{eq:continuous-path-iterate-bound}
\end{equation}
The corrections therefore have a summable uniform bound on {$v\in[0,\tau]$}.
Thus the iterates converge uniformly to a solution $f$. {For uniqueness, suppose $f$ and $g$ are two solutions. Applying the same integral estimate to their difference gives $|f(v;\rho)-g(v;\rho)|\le\gamma\int_0^v\max_{u\le s}|f(u;\rho)-g(u;\rho)|\,\mathrm ds$. Repeatedly applying this estimate gives $\max_{s\le\tau}|f(s;\rho)-g(s;\rho)|\le((\gamma\tau)^n/n!)\max_{s\le\tau}|f(s;\rho)-g(s;\rho)|$. Letting $n\to\infty$ proves uniqueness.}

\vspace{5pt}
\noindent\textbf{Continuous change over one time step.} Summing the bounds on successive differences in~\eqref{eq:continuous-path-iterate-bound} and using $f^{[0]}(v;\rho)=1$, we have $|f(s;\rho)|\le 1+\sum_{n=0}^{\infty}{(\gamma s)^{n+1}}/{(n+1)!}=e^{\gamma s}$ for $s\in[0,\tau]$. Since the derivative $\partial_s f(s;\rho)$ is a weighted average of $f(u;\rho)$, we have $|\partial_s f(s;\rho)|\le\gamma\max_{u\le s}|f(u;\rho)|\le\gamma e^{\gamma\tau}$ on $[0,\tau]$. The mean value theorem then gives for $s,s'\in[0,\tau]$,
\begin{equation}
 |f(s;\rho)-f(s';\rho)|\le\gamma e^{\gamma\tau}|s-s'|.
 \label{eq:continuous-path-lipschitz}
\end{equation}
To control the second derivative, we differentiate~\eqref{eq:continuous-path-definition} and use~\eqref{eq:continuous-path-lipschitz} to obtain, for $0<s\le\tau$,
\begin{equation}
|\partial_s^2 f(s;\rho)|
\le\frac{\gamma\int_0^s e^{-(s-u)/\rho}
|f(u;\rho)-f(s;\rho)|\,\mathrm du}{\rho^2(1-e^{-s/\rho})^2}\le
 \frac{\gamma^2e^{\gamma\tau}(1-(1+s/\rho)e^{-s/\rho})}{(1-e^{-s/\rho})^2}
\le\gamma^2e^{\gamma\tau}.
\end{equation}
The last inequality uses $1-s/\rho\le e^{-s/\rho}$. Therefore for $t<T$, the continuous change has main term $\eta\partial_s f(t\eta;\rho)$ and error bounded by
\begin{equation}
\begin{aligned}
|f((t+1)\eta;\rho)-f(t\eta;\rho)
-\eta\partial_s f(t\eta;\rho)|&=\Big|\int_{t\eta}^{(t+1)\eta}
[\partial_s f(s;\rho)-\partial_s f(t\eta;\rho)]\,\mathrm ds\Big|\\
&\le\gamma^2e^{\gamma\tau}
\int_{t\eta}^{(t+1)\eta}(s-t\eta)\,\mathrm ds
=\frac12\gamma^2e^{\gamma\tau}\eta^2.
\end{aligned}
\label{eq:continuous-step-error}
\end{equation}
The same estimate holds for $t=0$ by continuity of $\partial_s f(s;\rho)$ at $s=0$.

\vspace{5pt}
\noindent\textbf{Comparison of the discrete and continuous averages.}
We next control $|(f_{t+1}(\beta)-f_t(\beta))/\eta-\partial_s f(t\eta;\rho)|$.
For $1\le t<T$, set $r=-\eta/\log\beta$, so $e^{-\eta/r}=\beta$ and $\rho-\eta\le r\le\rho$. We can write the discrete average as
\[
 \sum_{t'=0}^t\omega_\beta(t,t')f_{t'}(\beta)
 =\sum_{k=0}^t\int_{k\eta}^{(k+1)\eta}
 \frac{e^{-a/r}f_{t-k}(\beta)}{r(1-e^{-(t+1)\eta/r})}\,\mathrm d a.
\]
Using the definition of $f_t(\beta)$ in Section~\ref{sec:exact} and~\eqref{eq:continuous-path-definition}, we can decompose the difference between the discrete and continuous changes as
\begin{equation}
\begin{aligned}
 &|(f_{t+1}(\beta)-f_t(\beta))/\eta-\partial_s f(t\eta;\rho)|\\
 &\quad\le\gamma\max_{t'\le t}|f_{t'}(\beta)-f(t'\eta;\rho)|\\
 &\qquad+\gamma\underbrace{
 \bigg|\sum_{k=0}^t\int_{k\eta}^{(k+1)\eta}
 \frac{e^{-a/r}f((t-k)\eta;\rho)}{r(1-e^{-(t+1)\eta/r})}\,\mathrm d a-\int_0^{t\eta}\frac{e^{-a/r}f(t\eta-a;\rho)}{r(1-e^{-t\eta/r})}\mathrm d a\bigg|
 }_{(\mathrm I)}\\
 &\qquad+\gamma\underbrace{\bigg|\int_0^{t\eta}f(t\eta-a;\rho)
 \bigg[\frac{e^{-a/r}}{r(1-e^{-t\eta/r})}
 -\frac{e^{-a/\rho}}{\rho(1-e^{-t\eta/\rho})}\bigg]\,\mathrm d a\bigg|}_{(\mathrm{II})}.
\end{aligned}
\label{eq:continuous-rate-decomposition}
\end{equation}
The first term bounds the error caused by replacing $f_{t'}(\beta)$ with $f(t'\eta;\rho)$. Term $(\mathrm I)$ compares the grid values of $f(\cdot;\rho)$ with its values inside each interval and accounts for the different normalizations on $[0,t\eta]$ and $[0,(t+1)\eta]$. Term $(\mathrm{II})$ accounts for the change of weight parameter from $r$ to $\rho$ while keeping $f(\cdot;\rho)$ fixed. Here we set $a=t\eta-u$ to match the continuous expression in~\eqref{eq:continuous-path-definition} with the discrete average.
Splitting off the last interval in $(\mathrm I)$ gives
\begin{equation}
\begin{aligned}
 (\mathrm I)&\le\sum_{k=0}^{t-1}\int_{k\eta}^{(k+1)\eta}
 \big|f((t-k)\eta;\rho)-f(t\eta-a;\rho)\big|
 \frac{e^{-a/r}}{r(1-e^{-(t+1)\eta/r})}\,\mathrm d a\\
 &\qquad+\omega_\beta(t,0)\int_0^{t\eta}
 \big|f(0;\rho)-f(t\eta-a;\rho)\big|
 \frac{e^{-a/r}}{r(1-e^{-t\eta/r})}\,\mathrm d a\\
 &\quad\le\gamma e^{\gamma\tau}\eta
 \frac{1-e^{-t\eta/r}}{1-e^{-(t+1)\eta/r}}
 +\omega_\beta(t,0)\gamma e^{\gamma\tau}t\eta
 \le2\gamma e^{\gamma\tau}\eta.
\end{aligned}
\label{eq:continuous-average-comparison-I}
\end{equation}
The first inequality uses $(1-e^{-t\eta/r})/(1-e^{-(t+1)\eta/r})=1-\omega_\beta(t,0)$ to account for the different normalizations. The second applies~\eqref{eq:continuous-path-lipschitz} with input differences at most $\eta$ and $t\eta$, respectively. The last inequality uses $\omega_\beta(t,0)=\beta^t/\sum_{k=0}^t\beta^k\le1/(t+1)$. To bound $(\mathrm{II})$, differentiate with respect to the weight parameter while keeping $f(\cdot;\rho)$ fixed. For $x>0$, we define the weighted mean of $a$ as $m(x):=\int_{0}^{t\eta}(ae^{-a/x})/(x(1-e^{-t\eta/x}))\,\mathrm da=x-t\eta/(e^{t\eta/x}-1)\in[0,t\eta]$. So the derivative satisfies
\begin{equation}
\begin{aligned}
 &\Big|\frac{\partial}{\partial x}
 \int_0^{t\eta}f(t\eta-a;\rho)
 \frac{e^{-a/x}}{x(1-e^{-t\eta/x})}\,\mathrm d a\Big|\\
 &\quad=\Big|\frac1{x^2}\int_0^{t\eta}
 f(t\eta-a;\rho)(a-m(x))
 \frac{e^{-a/x}}{x(1-e^{-t\eta/x})}\,\mathrm da\Big|\\
 &\quad=\Big|\frac1{x^2}\int_0^{t\eta}
 \big[f(t\eta-a;\rho)-f(t\eta-m(x);\rho)\big](a-m(x))
 \frac{e^{-a/x}}{x(1-e^{-t\eta/x})}\,\mathrm d a\Big|\\
 &\quad\le\frac{\gamma e^{\gamma\tau}}{x^2}
 \int_0^{t\eta}(a-m(x))^2
 \frac{e^{-a/x}}{x(1-e^{-t\eta/x})}\,\mathrm d a=\gamma e^{\gamma\tau}
 \Big[1-\frac{(t\eta/x)^2e^{t\eta/x}}{(e^{t\eta/x}-1)^2}\Big]
 \le\gamma e^{\gamma\tau}.
\end{aligned}
\end{equation}
The second equality uses the zero weighted mean of $a-m(x)$, and the first inequality applies~\eqref{eq:continuous-path-lipschitz}.
Integrating this derivative bound from $r$ to $\rho$ gives $(\mathrm{II})\le\gamma e^{\gamma\tau}(\rho-r)\le\gamma e^{\gamma\tau}\eta$. Substituting the bounds for $(\mathrm I)$ and $(\mathrm{II})$ into~\eqref{eq:continuous-rate-decomposition} gives
\begin{equation}
 \begin{aligned}
 &|(f_{t+1}(\beta)-f_t(\beta))/\eta-\partial_s f(t\eta;\rho)|\le\gamma\max_{t'\le t}|f_{t'}(\beta)-f(t'\eta;\rho)|
 +3\gamma^2e^{\gamma\tau}\eta.
 \end{aligned}
 \label{eq:continuous-average-comparison}
\end{equation}
At $t=0$, the left-hand side equals zero, so this bound also holds there.
The triangle inequality together with~\eqref{eq:continuous-step-error} and~\eqref{eq:continuous-average-comparison} gives, for $t<T$,
\begin{equation}
\begin{aligned}
|f_{t+1}(\beta)-f((t+1)\eta;\rho)|&\le |f_t(\beta)-f(t\eta;\rho)|
+\eta\Big|
\frac{f_{t+1}(\beta)-f_t(\beta)}{\eta}
-\partial_s f(t\eta;\rho)\Big|\\
&\qquad+|f((t+1)\eta;\rho)-f(t\eta;\rho)
-\eta\partial_s f(t\eta;\rho)|\\
&\le(1+\gamma\eta)\max_{t'\le t}|f_{t'}(\beta)-f(t'\eta;\rho)|
 +7\gamma^2e^{\gamma\tau}\eta^2/2.
\end{aligned}
\label{eq:continuous-error-propagation}
\end{equation}
Since $f_0(\beta)=f(0;\rho)=1$, repeatedly applying~\eqref{eq:continuous-error-propagation} gives
\begin{equation}
 \max_{t\le T}|f_t(\beta)-f(t\eta;\rho)|
 \le\frac72\gamma^2e^{\gamma\tau}\eta^2\sum_{t=0}^{T-1}(1+\gamma\eta)^t=\frac72\gamma e^{\gamma\tau}\eta\big[(1+\gamma\eta)^T-1\big]
 \le\frac72\gamma e^{2\gamma\tau}\eta.
 \end{equation}
The last inequality uses $(1+\gamma\eta)^T-1\le e^{\gamma T\eta}\le e^{\gamma\tau}$.
\end{proof}
\section{Proof of Theorem~\ref{thm:smaller-lr}}
\label{app:smaller-lr}
\begin{proof}
We first match the exact endpoints at finite step size. The mean comparison then follows from the sampling expansion; the variance comparison requires the continuous response and its limit at long optimization times. In this section, $f_T(\eta,\beta)$ displays the learning-rate argument because the two runs use different rates.

\vspace{5pt}
\noindent\textbf{Finite-step matching and expected objectives.}
The assumptions give $0<\eta_1\le\rho_1<\rho_2\le1/(4\gamma)$ and $ 0\le\beta_1<1-\eta_1/\rho_2 \le1-4\gamma\eta_1 \le((1-\gamma\eta_1)/(1+\gamma\eta_1))^2. $ By the residual comparison in the proof of Theorem~\ref{thm:exact-benefit}, $f_T(\eta_1,1-\eta_1/\rho_2)<f_T(\eta_1,\beta_1)<1$. For fixed $T$, the definition of $f_t$ in Section~\ref{sec:exact} gives $f_T(\eta,1-\eta/\rho_2)\to1$ as $\eta\to0$. Since $f_T(\eta,1-\eta/\rho_2)$ is continuous for $\eta\in(0,\eta_1]$, there exists a deterministic $\eta_2\in(0,\eta_1)$ such that
\[
 f_T(\eta_2,1-\eta_2/\rho_2)=f_T(\eta_1,\beta_1).
\]
Set $\beta_2=1-\eta_2/\rho_2$. Then $\beta_2>\beta_1$, and Lemma~\ref{lem:exact-dynamics} gives $\theta_T^{(2)}=\theta_T^{(1)}$. Applying the individual expansion~\eqref{eq:locality-expansion} to both runs and using $\theta_T^{(2)}=\theta_T^{(1)}$ proves the expected objective comparison.

\vspace{5pt}
\noindent\textbf{Limiting optimization times.}
For the variance comparison, we use the continuous approximation from Lemma~\ref{lem:continuous-path}. We first identify the limiting optimization time of the matched run. Fix $\tau,\rho_1,\rho_2$. By~\eqref{eq:continuous-safe-path}, $f(\cdot;\rho_2)$ is continuous, strictly decreasing from one, and tends to zero. Hence there is a unique $\tau_2>0$ such that
\begin{equation}
 f(\tau_2;\rho_2)=f(\tau;\rho_1),\qquad
 \theta(\tau_2;\rho_2)=\theta(\tau;\rho_1).
 \label{eq:continuous-endpoint-match}
\end{equation}
Combining Lemma~\ref{lem:continuous-path} with the exact endpoint matching gives $f(T\eta_2;\rho_2)=f(\tau;\rho_1)+O(\eta_1) =f(\tau_2;\rho_2)+O(\eta_1)$. Since $f(\cdot;\rho_2)$ is continuous and strictly decreasing, $f(T\eta_2;\rho_2)\to f(\tau_2;\rho_2)$ implies
\begin{equation}
 T\eta_2\to\tau_2,\qquad
 \frac{\eta_2}{\eta_1}=\frac{T\eta_2}{\tau}
 \to\frac{\tau_2}{\tau}
 \quad\text{as }\eta_1\to0.
 \label{eq:matched-time-limit}
\end{equation}
To compare the variance coefficients, we introduce the continuous counterpart of the noise response in~\eqref{eq:response}. Here $s=t\eta$ is optimization time. Fix $\rho$ and a noise input time $s_1>0$. For a unit noise input in one coordinate at update $t'$, where $t'\eta\to s_1$ as $\eta\to0$, let $\Delta\theta_t^\beta$ and $\Delta m_t^\beta$ denote the discrete differences in that parameter coordinate and its EMA momentum before bias correction from the dynamics without this input. With $\beta=1-\eta/\rho$, define the continuous responses for $s_2>s_1$ by
\[
 \Delta\theta(s_2;\rho)=\lim_{\eta\to0}\Delta\theta_t^\beta/\eta,\qquad
 \Delta m(s_2;\rho)=\lim_{\eta\to0}\Delta m_t^\beta/\eta,
\]
where $t\eta\to s_2$ as $\eta\to0$. The updates in~\eqref{eq:npg-ema} and the difference dynamics in~\eqref{eq:noise-expansion} give
\[
 \partial_{s_2}\Delta\theta(s_2;\rho)=\Delta m(s_2;\rho)/(1-e^{-s_2/\rho}),\quad
 \partial_{s_2}\Delta m(s_2;\rho)=-(\gamma\Delta\theta(s_2;\rho)+\Delta m(s_2;\rho))/\rho.
\]
Let $\Phi_\rho(s_2,s_1)$ be the transition matrix for $(\Delta\theta(s_2;\rho),\Delta m(s_2;\rho))^\top$. It is defined by
\begin{equation}
 \partial_{s_2}\Phi_\rho(s_2,s_1)=
 \begin{pmatrix}0&(1-e^{-s_2/\rho})^{-1}\\-\gamma/\rho&-1/\rho\end{pmatrix}
 \Phi_\rho(s_2,s_1),\qquad \Phi_\rho(s_1,s_1)=\mathbf I_2.
 \label{eq:continuous-response}
\end{equation}
The initial condition is $\Delta\theta(s_1;\rho)=0$ and $\Delta m(s_1;\rho)=1/\rho$. Thus, the continuous counterpart of the normalized discrete response $\Psi_\beta(T,t)/\eta$ is
\[
 \Psi(s_2,s_1;\rho):=\Delta\theta(s_2;\rho)=\mathbf e_1^\top\Phi_\rho(s_2,s_1)\mathbf e_2/\rho.
\]
Here $\mathbf e_1=(1,0)^\top$ selects the parameter component and $\mathbf e_2/\rho=(0,1)^\top/\rho$ represents the initial state. We define the continuous-time counterpart of the covariance coefficient $\Sigma_T^\beta$ in~\eqref{eq:discrete-covariance} by
\begin{equation}
 \Sigma(s_2;\rho)
 =\int_0^{s_2}\Psi(s_2,s_1;\rho)^2
 \bigg[\operatorname{diag}\bigg(
 \frac{\mathbf r}{\pi_{\theta(s_1;\rho)}(\cdot\mid x)}\bigg)
 -\mathbf r\mathbf r^\top\bigg]\mathrm d s_1.
 \label{eq:continuous-covariance}
\end{equation}
As the optimization time $s_2\to\infty$, evaluating the limiting integral in~\eqref{eq:continuous-covariance} gives
\[
 \Sigma_\infty=\frac1{2\gamma}
 \bigg[\operatorname{diag}\bigg(
 \frac{\mathbf r}{\pi_{\theta^\star}(\cdot\mid x)}\bigg)
 -\mathbf r\mathbf r^\top\bigg].
\]
The following lemma establishes convergence of the covariance coefficients and provides uniform relative error control. Its proof is given in the next subsection.
\begin{lemma}[Limits of the covariance coefficient]
\label{lem:continuous-covariance}
Fix $0<\rho\le1/(4\gamma)$. If $\eta\to0$, $T\eta\to s>0$, and $\beta=1-\eta/\rho$, then $\Sigma_T^\beta\to\Sigma(s;\rho)$. Moreover,
\begin{equation}
 \Sigma(s;\rho)\longrightarrow\Sigma_\infty,
 \qquad
 \sup_{\mathbf v:\,\mathbf v^\top\Sigma_\infty\mathbf v>0}
 \Big|\frac{\mathbf v^\top\Sigma(s;\rho)\mathbf v}{\mathbf v^\top\Sigma_\infty\mathbf v}-1\Big|
 \longrightarrow0
 \quad(s\to\infty).
 \label{eq:relative-covariance-limit}
\end{equation}
\end{lemma}
\noindent\textbf{Variance ratio at matched endpoints.}
We first note that the gradients at $\theta_T^{(1)}$ and $\theta_T^{(2)}$ are identical because the exact endpoints satisfy $\theta_T^{(1)}=\theta_T^{(2)}$. For the first run, $T\eta_1=\tau$. Lemma~\ref{lem:continuous-path} and the representation of the exact path in Lemma~\ref{lem:exact-dynamics} give $\|\theta_T^{(1)}-\theta(\tau;\rho_1)\|_2=O(\eta_1)$. The continuity of $\nabla J(\theta)$ therefore yields, as $\eta_1\to0$,
\[
 \nabla J(\theta_T^{(1)})=\nabla J(\theta_T^{(2)})
 \to\nabla J(\theta(\tau;\rho_1)).
\]
According to Lemma~\ref{lem:continuous-covariance}, the two discrete covariance coefficients converge to their continuous counterparts. Thus, the quadratic forms in the variance expression in~\eqref{eq:uniform-individual-variance} for the two runs converge to $\nabla J(\theta(\tau;\rho_1))^\top \Sigma(\tau;\rho_1)\nabla J(\theta(\tau;\rho_1))$ and $\nabla J(\theta(\tau;\rho_1))^\top \Sigma(\tau_2;\rho_2)\nabla J(\theta(\tau;\rho_1))$, respectively.

To obtain the limiting variance ratio from the leading variance terms, we first verify these limiting quadratic forms are nonzero for all sufficiently large fixed $\tau$. As shown in the proof of Lemma~\ref{lem:continuous-covariance}, $\Sigma(\tau;\rho_1)$ and $\Sigma(\tau_2;\rho_2)$ are positive definite on the reward-$1$ subspace and zero on its orthogonal complement. It therefore suffices to find a nonzero reward-$1$ coordinate of the common limiting gradient. Along the exact path $\theta(\tau;\rho_1)=\theta^\star+f(\tau;\rho_1)(\theta_0-\theta^\star)$, fix any reward-$1$ response $y$. Since $f(\tau;\rho_1)>0$ and the policy has full support, \eqref{eq:objective-gradient-appendix} shows that its gradient coordinate vanishes exactly when $\theta_0(x,y)-\theta^\star(x,y)$ equals the average of $\theta_0-\theta^\star$ under $\pi_{\theta(\tau;\rho_1)}$. The derivative of this average with respect to $\tau$ is
\[
 \partial_\tau f(\tau;\rho_1)\,
 \operatorname{Var}_{y'\sim\pi_{\theta(\tau;\rho_1)}(\cdot\mid x)}
 [\theta_0(x,y')-\theta^\star(x,y')]<0,
\]
because $\partial_\tau f(\tau;\rho_1)<0$ and the variance is positive by nonoptimal initialization and full support. The average is therefore strictly decreasing, so the chosen gradient coordinate can vanish at most once for $\tau>0$ and is nonzero for all sufficiently large $\tau$. Both quadratic forms are consequently strictly positive for these $\tau$. Using $\eta_2/\eta_1\to\tau_2/\tau$ and~\eqref{eq:uniform-individual-variance}, we obtain the limiting variance ratio
\begin{equation}
 \lim_{\eta_1\to0}
 \frac{\operatorname{Var}[J(\widehat\theta_T^{(2)})]}
 {\operatorname{Var}[J(\widehat\theta_T^{(1)})]}
 =\frac{\tau_2}{\tau}
 \frac{\nabla J(\theta(\tau;\rho_1))^\top
 \Sigma(\tau_2;\rho_2)\nabla J(\theta(\tau;\rho_1))}
 {\nabla J(\theta(\tau;\rho_1))^\top
 \Sigma(\tau;\rho_1)\nabla J(\theta(\tau;\rho_1))}.
 \label{eq:matched-variance-ratio}
\end{equation}
The relative bound in Lemma~\ref{lem:continuous-covariance} makes the quotient of quadratic forms in~\eqref{eq:matched-variance-ratio} tend to one, even though the common gradient tends to zero. It remains to determine $\tau_2/\tau$.

\vspace{5pt}
\noindent\textbf{Long-time residual decay.}
To determine $\tau_2/\tau$, define the logarithmic decay rate $h_\rho(s)=-\partial_s\log f(s;\rho)$. Differentiating~\eqref{eq:continuous-path-definition} gives, for $s>0$,
\begin{equation}
 h_\rho'(s)=h_\rho(s)^2
 -(h_\rho(s)-\gamma)/[\rho(1-e^{-s/\rho})],
 \qquad h_\rho(0)=\gamma.
 \label{eq:continuous-log-rate}
\end{equation}
Define $\kappa(\rho)=2\gamma/(1+\sqrt{1-4\gamma\rho})$, the smaller root of $\rho\kappa^2-\kappa+\gamma=0$. We now show that $h_\rho(s)\to\kappa(\rho)$ as $s\to\infty$. At $h_\rho(s)=\kappa(\rho)$,~\eqref{eq:continuous-log-rate} gives $h_\rho'(s)=-\kappa(\rho)^2e^{-s/\rho}/(1-e^{-s/\rho})<0$. Since $h_\rho(0)=\gamma<\kappa(\rho)$ and $h_\rho(s)$ cannot reach $\kappa(\rho)$ from below, we have $h_\rho(s)<\kappa(\rho)$. Using~\eqref{eq:continuous-log-rate}, we obtain
\[
 h_\rho'(s)\ge(\kappa(\rho)-h_\rho(s))^2
  -\kappa(\rho)^2e^{-s/\rho}/(1-e^{-s/\rho}).
  \]
This uses $\kappa(\rho)-\gamma=\rho\kappa(\rho)^2$ and $2\rho\kappa(\rho)\le1$. For any $0<\varepsilon<\kappa(\rho)-\gamma$, the bound gives $h_\rho'(s)\ge\varepsilon^2/2$ whenever $h_\rho(s)\le\kappa(\rho)-\varepsilon$ and $s$ is sufficiently large. Hence $h_\rho(s)>\kappa(\rho)-\varepsilon$ for all sufficiently large $s$. With the upper bound $h_\rho(s)<\kappa(\rho)$, we have $h_\rho(s)\to\kappa(\rho)$ as $s\to\infty$.
Integrating $h_\rho(s)=-\partial_s\log f(s;\rho)$ and using $f(0;\rho)=1$ gives
\[
 -\log f(s_2;\rho)/s_2
 =(1/s_2)\int_0^{s_2}h_\rho(s_1)\,\mathrm d s_1
 \to\kappa(\rho)\qquad\text{as }s_2\to\infty.
\]
Taking logarithms in~\eqref{eq:continuous-endpoint-match} and applying the preceding limit to both endpoints gives
\[
 \begin{aligned}
 \lim_{\tau\to\infty}\tau_2/\tau
 &=\lim_{\tau\to\infty}
 \frac{-\log f(\tau;\rho_1)/\tau}
 {-\log f(\tau_2;\rho_2)/\tau_2}\\
 &=\kappa(\rho_1)/\kappa(\rho_2)
 =(1+\sqrt{1-4\gamma\rho_2})/(1+\sqrt{1-4\gamma\rho_1})<1.
 \end{aligned}
\]
Substituting into~\eqref{eq:matched-variance-ratio} proves the double-limit variance ratio in Theorem~\ref{thm:smaller-lr}.
\end{proof}
\subsection{Proof of Lemma~\ref{lem:continuous-covariance}}
\begin{proof}
We pass to the continuous covariance, compute its long-time limit, and control relative error.

\vspace{5pt}
\noindent\textbf{Discrete-to-continuous covariance.}
Fix $0<\rho\le1/(4\gamma)$ and set $\beta=1-\eta/\rho$. Let $\eta\to0$ with $T\eta\to s>0$, and fix $0<a<s$. The updates in~\eqref{eq:npg-ema} and~\eqref{eq:noise-expansion} give the following discrete counterpart of the transition matrix in~\eqref{eq:continuous-response}.
\begin{equation}
  \begin{pmatrix}
 1-\gamma\eta/\sum_{k=0}^{t}\beta^k&\rho\beta/\sum_{k=0}^{t}\beta^k\\
 -\gamma\eta/\rho&\beta
 \end{pmatrix}.
\end{equation}
For $a\le t\eta$, we compare it with the continuous one-step transition in~\eqref{eq:continuous-response}. Taylor expansion gives
\[
 \begin{aligned}
 \Phi_\rho((t+1)\eta,t\eta)
 &=\Phi_\rho(t\eta,t\eta)
 +\eta\partial_{s_2}\Phi_\rho(s_2,t\eta)\big|_{s_2=t\eta}
 +O_a(\eta^2)\\
 &=\begin{pmatrix}
 1&\eta/(1-e^{-t\eta/\rho})\\
 -\gamma\eta/\rho&1-\eta/\rho
 \end{pmatrix}
 +O_a(\eta^2).
 \end{aligned}
\]
For the discrete matrix, the first-row entries satisfy $1-\gamma\eta/\sum_{k=0}^{t}\beta^k=1-\gamma\eta^2/\big[\rho(1-\beta^{t+1})\big]=1+O_a(\eta^2)$ and $\rho\beta/\sum_{k=0}^{t}\beta^k=\eta\beta/(1-\beta^{t+1})=\eta/(1-e^{-t\eta/\rho})+O_a(\eta^2)$. The second row agrees exactly with the continuous expansion. Thus the discrete transition matrix equals $\Phi_\rho((t+1)\eta,t\eta)+O_a(\eta^2)$ in operator norm, where $O_a(\cdot)$ suppresses constants that may depend on $a$.
By~\eqref{eq:npg-ema} and~\eqref{eq:response}, the responses at iteration $t+1$ to a unit noise input in one coordinate at update $t$ are $\eta/\sum_{k=0}^{t}\beta^k$ for that parameter coordinate and $\eta/\rho$ for its EMA momentum before bias correction. After division by $\eta$, these responses are $1/\sum_{k=0}^{t}\beta^k$ and $1/\rho$, respectively. The preceding Taylor expansion gives $\Phi_\rho((t+1)\eta,t\eta)\mathbf e_2/\rho=(\eta/\big[\rho(1-e^{-t\eta/\rho})\big],1/\rho-\eta/\rho^2)^\top+O_a(\eta^2)$. Combining the $O_a(\eta)$ response error at $(t+1)\eta$ with the $O_a(\eta^2)$ error in each subsequent transition gives
\begin{equation}
 \max_{a\le t\eta<T\eta}|\Psi_\beta(T,t)/\eta-\Psi(T\eta,t\eta;\rho)|\le C_a\eta(1+C_a\eta)^T+C_a\eta^2\sum_{k=0}^{T-1}(1+C_a\eta)^k=O_a(\eta).
 \label{eq:continuous-response-approximation}
\end{equation}
Each one-step transition has operator norm at most $1+C_a\eta$. The first term bounds the initial response differences after propagation to $T\eta$. The sum bounds the contribution of the subsequent transition errors. The final bound uses $(1+C_a\eta)^T\le e^{C_aT\eta}$ and $T\eta\to s$.
The above bounds give $|\Psi_\beta(T,t)/\eta-\Psi(s,s_1;\rho)|\to0$ as $T\eta\to s$, $t\eta\to s_1>0$. Taking the limit in the first inequality in~\eqref{eq:response-bounds} gives, for $s_1\leq s$,
\begin{equation}
 |\Psi(s,s_1;\rho)|\le e^{\gamma s}[1+\log(s/s_1)].
 \label{eq:continuous-response-envelope}
\end{equation}
We first consider the covariance contribution over $[a,s]$. Lemmas~\ref{lem:continuous-path} and~\ref{lem:exact-dynamics} give $\max_{t\le T}\|\theta_t^\beta-\theta(t\eta;\rho)\|_2=O(\eta)$. The safe-range residual bounds in the proof of Theorem~\ref{thm:exact-benefit} and~\eqref{eq:continuous-safe-path} also give $\|\theta_t^\beta\|_\infty,\|\theta(s_1;\rho)\|_\infty \le\|\theta^\star\|_\infty+\|\theta_0-\theta^\star\|_\infty$. Using~\eqref{eq:discrete-covariance}, we obtain
\[
 \begin{aligned}
 &\eta\sum_{t=\lceil a/\eta\rceil}^{T-1}(\Psi_\beta(T,t)/\eta)^2
 \big[\operatorname{diag}(\mathbf r/\pi_{\theta_t^\beta}(\cdot\mid x))-\mathbf r\mathbf r^\top\big]\\
 &\quad=\eta\sum_{t=\lceil a/\eta\rceil}^{T-1}\Psi(T\eta,t\eta;\rho)^2
 \big[\operatorname{diag}(\mathbf r/\pi_{\theta(t\eta;\rho)}(\cdot\mid x))-\mathbf r\mathbf r^\top\big]+O_a(\eta).
 \end{aligned}
\]
The equality uses~\eqref{eq:continuous-response-approximation} to replace $\Psi_\beta(T,t)/\eta$ by $\Psi(T\eta,t\eta;\rho)$ and the parameter approximation above to replace $\theta_t^\beta$ by $\theta(t\eta;\rho)$. The bounded parameter entries give uniformly bounded covariance matrices. Applying the analysis used to obtain~\eqref{eq:covariance-replacement-error} to the parameter approximation above gives a covariance replacement error of $O(\eta)$. Together with~\eqref{eq:continuous-response-envelope}, these estimates give an error of $O_a(\eta)$ in each summand and a total error of $\eta T O_a(\eta)=O_a(\eta)$. The sum on the right is a Riemann sum for the continuous integrand in~\eqref{eq:continuous-covariance}, so it converges to the integral over $[a,s]$ as $\eta\to0$ and $T\eta\to s$.
For the omitted early samples, the first inequality in~\eqref{eq:response-bounds} gives
\[
 \limsup_{\eta\to0}\eta\sum_{t=0}^{\lceil a/\eta\rceil-1}
 |\Psi_\beta(T,t)/\eta|^2
 \le e^{2\gamma s}\int_0^a[1+\log(s/s_1)]^2\,\mathrm d s_1.
\]
For $t\ge1$, the inequality $\sum_{k=t+1}^T1/k\le\log(T/t)$ compares the sum with the integral. The $t=0$ term $\eta|\Psi_\beta(T,0)/\eta|^2$ is $O(\eta\log^2(1/\eta))$ and tends to zero. Equation~\eqref{eq:continuous-response-envelope} bounds the continuous contribution $\int_0^a|\Psi(s,s_1;\rho)|^2\,\mathrm d s_1$ by the same integral, which tends to zero as $a\to0$. Together with the uniform reward covariance bound, these estimates make both omitted contributions arbitrarily small by letting $a\to0$ after taking $\eta\to0$. Combining this with the comparison on $[a,s]$ gives $\Sigma_T^\beta\to\Sigma(s;\rho)$ in operator norm as $\eta\to0$ and $T\eta\to s$.

\vspace{5pt}
\noindent\textbf{Long-time limit of the covariance coefficient.}
We establish the existence of $\lim_{s\to\infty}\Sigma(s;\rho)$ and compute its value.
The coefficient matrix in~\eqref{eq:continuous-response} can be written as
\[
 \begin{pmatrix}0&(1-e^{-s/\rho})^{-1}\\-\gamma/\rho&-1/\rho\end{pmatrix}
 =\begin{pmatrix}0&1\\-\gamma/\rho&-1/\rho\end{pmatrix}+\begin{pmatrix}0&e^{-s/\rho}/(1-e^{-s/\rho})\\0&0\end{pmatrix}.
 \label{eq:continuous-response-decomposition}
\]
Let $\mathbf K_\rho$ denote the first matrix and $\mathbf R_\rho(s)$ denote the remainder. For $s\ge1$, the remainder's operator norm satisfies $\|\mathbf R_\rho(s)\|_{\mathrm{op}}=e^{-s/\rho}/(1-e^{-s/\rho})\le e^{-s/\rho}/(1-e^{-1/\rho})$. Both eigenvalues of $\mathbf K_\rho$ have negative real parts, so $\|e^{\mathbf K_\rho v}\|_{\mathrm{op}}\le C_0e^{-cv}$ for some $C_0\ge1$ and $c>0$ depending only on $\rho$ and $\gamma$. Variation of constants in~\eqref{eq:continuous-response} gives
\[
 \Phi_\rho(s,s_1)=e^{\mathbf K_\rho(s-s_1)}
 +\int_{s_1}^s e^{\mathbf K_\rho(s-u)}\mathbf R_\rho(u)\Phi_\rho(u,s_1)\,\mathrm d u.
\]
Taking operator norms and using the bound on $\|e^{\mathbf K_\rho v}\|_{\mathrm{op}}$ yields, for $s\ge s_1\ge1$,
\begin{align}
   \|\Phi_\rho(s,s_1)\|_{\mathrm{op}}&\le C_0e^{-c(s-s_1)}+C_0e^{-c(s-s_1)}\int_{s_1}^s\|\mathbf R_\rho(u)\|_{\mathrm{op}}e^{c(u-s_1)}\|\Phi_\rho(u,s_1)\|_{\mathrm{op}}\,\mathrm d u\notag\\
  &\le C_0e^{-c(s-s_1)}
 \exp\bigg(C_0\int_{s_1}^s\|\mathbf R_\rho(u)\|_{\mathrm{op}}\,\mathrm d u\bigg)\notag\\
  &\le Ce^{-c(s-s_1)}.\label{eq:exponential-response}
\end{align}
The second inequality uses Gronwall's inequality. The next inequality uses $\int_{s_1}^s\|\mathbf R_\rho(u)\|_{\mathrm{op}}\,\mathrm d u\le\int_1^\infty e^{-u/\rho}/(1-e^{-1/\rho})\,\mathrm d u=\rho e^{-1/\rho}/(1-e^{-1/\rho})$. Therefore, there exists a constant $C$ depending only on $\rho$ and $\gamma$ such that $\|\Phi_\rho(s,s_1)\|_{\mathrm{op}}\le Ce^{-c(s-s_1)}$ for all $s\ge s_1\ge1$.
The lower bound $s_1\ge1$ keeps this integral uniformly bounded, whereas $\|\mathbf R_\rho(u)\|_{\mathrm{op}}\sim\rho/u$ as $u\to0$.
For $0<s_1<1$, \eqref{eq:continuous-response-envelope} bounds the parameter response on $[s_1,1]$ by $e^\gamma[1+\log(1/s_1)]$. Integrating the second row of~\eqref{eq:continuous-response} with $\Delta m(s_1;\rho)=1/\rho$ gives $|\Delta m(1;\rho)|\le\rho^{-1}+\gamma e^\gamma[1+\log(1/s_1)]$. Using $\Phi_\rho(s,s_1)=\Phi_\rho(s,1)\Phi_\rho(1,s_1)$, \eqref{eq:exponential-response}, and the uniform reward covariance bound, the operator norm of the contribution of these inputs to~\eqref{eq:continuous-covariance} is at most $C_1e^{-2c(s-1)}\int_0^1[1+\log(1/s_1)]^2\,\mathrm d s_1$. Here $C_1>0$ is a constant independent of $s$ and $s_1$. As $s\to\infty$, this contribution tends to zero, so the inputs with $s_1\in[0,1]$ can be omitted when evaluating the covariance limit.

For inputs with $s_1\ge1$, write $v=s-s_1$ and define $g(v)=\mathbf e_1^\top e^{\mathbf K_\rho v}\mathbf e_2/\rho$. For fixed $v\ge0$, the decomposition in~\eqref{eq:continuous-response-decomposition} gives $\sup_{u\in[s-v,s]}\|\mathbf R_\rho(u)\|_{\mathrm{op}}\to0$. The variation-of-constants formula above then yields $\Psi(s,s-v;\rho)\to g(v)$. Equation~\eqref{eq:continuous-safe-path} also gives $\theta(s-v;\rho)\to\theta^\star$. Extend the integrand in~\eqref{eq:continuous-covariance} after this change of variables by zero for $v>s-1$. By~\eqref{eq:exponential-response} and the uniform reward covariance bound established above, its operator norm is at most $C_1e^{-2cv}$, with $C_1$ chosen to cover both covariance estimates. Since this bound is integrable on $[0,\infty)$, dominated convergence allows the limit and integral to be interchanged. Together with the vanishing contribution of $0<s_1<1$, this gives
\[
 \lim_{s\to\infty}\Sigma(s;\rho)
 =\bigg(\int_0^\infty g(v)^2\,\mathrm d v\bigg)
 \big[\operatorname{diag}\big(\mathbf r/\pi_{\theta^\star}(\cdot\mid x)\big)
 -\mathbf r\mathbf r^\top\big].
\]
Integration by parts in $\rho g''+g'+\gamma g=0$ with $g(0)=0$ and $g'(0)=1/\rho$ gives $\int_0^\infty g(v)^2\,\mathrm d v=1/(2\gamma)$, so this limit equals $\Sigma_\infty$.

\vspace{5pt}
\noindent\textbf{Relative covariance bound.}
For a vector $\mathbf v$, let $v_y$ denote its entry corresponding to response $y$. By the covariance interpretation in~\eqref{eq:discrete-covariance}, the null space of a reward covariance matrix consists of vectors $\mathbf v$ for which $r(x,y)v_y/\pi_\theta(y\mid x)$ is constant over $y$. Full support and the existence of a reward-$0$ response force this constant to be zero, which is equivalent to $v_y=0$ on all reward-$1$ coordinates. Thus these matrices are positive definite on the span of $\{\mathbf e_y:r(x,y)=1\}$ and vanish on its orthogonal complement.
With initial state $\mathbf e_2/\rho$, the first row of~\eqref{eq:continuous-response} gives $\Psi(s,s_1;\rho)>0$ for $s_1<s$ sufficiently close to each fixed $s>0$. Equation~\eqref{eq:continuous-covariance} therefore gives $\Sigma(s;\rho)$ the same positive definiteness and null space as $\Sigma_\infty$. Let $\lambda_*>0$ denote the smallest eigenvalue of $\Sigma_\infty$ restricted to this subspace. The relative error satisfies
\[
 \sup_{\mathbf v:\,\mathbf v^\top\Sigma_\infty\mathbf v>0}
 \big|\mathbf v^\top\Sigma(s;\rho)\mathbf v/(\mathbf v^\top\Sigma_\infty\mathbf v)-1\big|
 \le\|\Sigma(s;\rho)-\Sigma_\infty\|_{\mathrm{op}}/\lambda_*
 \to0.
\]
Removing the reward-$0$ coordinates changes neither quadratic form, so the supremum can be restricted to the reward-$1$ subspace. On this subspace, the absolute difference of the quadratic forms is at most $\|\Sigma(s;\rho)-\Sigma_\infty\|_{\mathrm{op}}\cdot\|\mathbf v\|_2^2$, while $\mathbf v^\top\Sigma_\infty\mathbf v\ge\lambda_*\|\mathbf v\|_2^2$.
\end{proof}
\section{Experimental details}
\label{app:experimental-details}
\begin{figure}[!t]
 \centering
 \subcapcentertrue
 \subfigure[Gain and Noise.\label{fig:bandit-five-uniform-gain-noise}]{\includegraphics[width=0.24\textwidth]{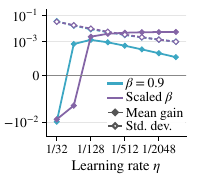}}\hfill
 \subfigure[Win Probability.\label{fig:bandit-five-uniform-win}]{\includegraphics[width=0.24\textwidth]{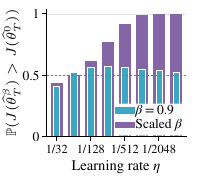}}\hfill
 \subfigure[Gain and Noise.\label{fig:bandit-five-random-gain-noise}]{\includegraphics[width=0.24\textwidth]{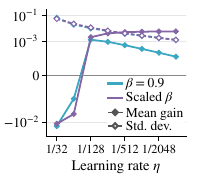}}\hfill
 \subfigure[Win Probability.\label{fig:bandit-five-random-win}]{\includegraphics[width=0.24\textwidth]{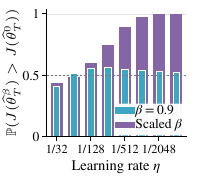}}
 \caption{Numerical bandit results with $5$ reward-one and $5$ reward-zero arms
 under uniform initialization (\ref{fig:bandit-five-uniform-gain-noise}, \ref{fig:bandit-five-uniform-win}) and random policy initialization (\ref{fig:bandit-five-random-gain-noise}, \ref{fig:bandit-five-random-win}).}
 \label{fig:bandit-five-results}
\end{figure}

\subsection{Bandit results}
\label{app:additional-bandit}

\noindent\textbf{Settings.}
We consider a 10-armed bandit problem and maximize
\begin{equation*}
 J(\theta)=\sum_{y=1}^{10}\pi_\theta(y\mid x)r(x,y)
 -\mathrm{KL}\left(\pi_\theta(\cdot\mid x)\Vert
 \pi_{\theta_0}(\cdot\mid x)\right).
\end{equation*}
We set $\gamma=1$, $B=1$,
and $\tau=1$, use $\eta=1/T$ for $T=32,64,\ldots,4096$, and compare
fixed $\beta=0.9$ and scaled $\beta=1-4\eta$ with $\beta=0$.
Each configuration uses $50{,}000$ independent training replicates,
with $J(\theta)$ evaluated exactly after training.
We vary the reward vector and policy initialization across four settings:
\begin{itemize}[leftmargin=*,itemsep=3pt]
 \item \textbf{3:7, uniform.} Reward vector
 $(1,1,1,0,0,0,0,0,0,0)$; initial policy
 $\pi_{\theta_0}(\cdot\mid x)=(0.1,\ldots,0.1)$.
 Results appear in Figures~\ref{fig:bandit-uniform-gain-noise}
 and~\ref{fig:bandit-uniform-win}.
 \item \textbf{3:7, random.} Reward vector
 $(1,1,1,0,0,0,0,0,0,0)$; initial policy
$\pi_{\theta_0}(\cdot\mid x)\approx (0.0388,\allowbreak 0.0713,\allowbreak 0.1502,\allowbreak 0.3932,\allowbreak 0.0100,\allowbreak 0.0459,\allowbreak 0.0635,\allowbreak 0.0928,\allowbreak 0.0891,\allowbreak 0.0453)$.
 Results appear in Figures~\ref{fig:bandit-random-gain-noise}
 and~\ref{fig:bandit-random-win}.
 \item \textbf{5:5, uniform.} Reward vector
 $(1,1,1,1,1,0,0,0,0,0)$; initial policy
 $\pi_{\theta_0}(\cdot\mid x)=(0.1,\ldots,0.1)$.
 Results appear in Figures~\ref{fig:bandit-five-uniform-gain-noise}
 and~\ref{fig:bandit-five-uniform-win}.
 \item \textbf{5:5, random.} Reward vector
 $(1,1,1,1,1,0,0,0,0,0)$; initial policy
 $\pi_{\theta_0}(\cdot\mid x)\approx (0.1320,\allowbreak 0.1596,\allowbreak 0.0361,\allowbreak 0.0342,\allowbreak 0.0221,\allowbreak 0.0846,\allowbreak 0.2399,\allowbreak 0.0667,\allowbreak 0.1895,\allowbreak 0.0353)$.
 Results appear in Figures~\ref{fig:bandit-five-random-gain-noise}
 and~\ref{fig:bandit-five-random-win}.
\end{itemize}
Random initializations are drawn separately from $\operatorname{Dirichlet}(1,\ldots,1)$ and are held fixed across learning rates, EMA coefficients, and training replicates.

\vspace{5pt}
\noindent\textbf{Results.}
All four settings exhibit the same trends as $\eta$ decreases:
the standard deviation scales as $\sqrt{\eta}$ under both schedules,
while the mean gain decays as $\eta$ with fixed $\beta$ and approaches
a positive constant with scaled $\beta$.
Consequently, the winning probability approaches $1/2$ in the fixed
setting and $1$ in the scaled setting.
These results support the robustness of the theoretical predictions
across reward vectors and policy initializations.
Table~\ref{tab:bandit-snr} reports the SNR defined in
Section~\ref{sec:numerical-bandit} for each setting and learning rate.

\begin{table}[!htbp]
 \centering
 \vspace{6pt}
 \begingroup
 \small
 \renewcommand{\arraystretch}{1.12}
 \setlength{\tabcolsep}{3pt}
 {
 \begin{tabular*}{\textwidth}{@{\extracolsep{\fill}}ll*{8}{r}@{}}
 \toprule
 \multirow{2}{*}{Setting} & \multirow{2}{*}{$\beta$}
 & \multicolumn{8}{c}{Learning rate $\eta$} \\
 \cmidrule(l){3-10}
 & & $1/32$ & $1/64$ & $1/128$ & $1/256$ & $1/512$
 & $1/1024$ & $1/2048$ & $1/4096$ \\
 \midrule
 \multirow{2}{*}{3:7, uniform} & $0.9$ & 0.1271 & 0.1236 & 0.1732 & 0.1530 & 0.1183 & 0.0876 & 0.0632 & 0.0453 \\
 & $1-4\eta$ & 0.0719 & 0.1118 & 0.3720 & 0.7229 & 1.1844 & 1.7900 & 2.6258 & 3.7647 \\
 \addlinespace[4pt]
 \multirow{2}{*}{3:7, random} & $0.9$ & 0.1809 & 0.0904 & 0.1536 & 0.1393 & 0.1071 & 0.0796 & 0.0573 & 0.0406 \\
 & $1-4\eta$ & 0.1220 & 0.0585 & 0.3178 & 0.6486 & 1.0712 & 1.6238 & 2.3713 & 3.3982 \\
 \addlinespace[4pt]
 \multirow{2}{*}{5:5, uniform} & $0.9$ & 0.2625 & 0.0402 & 0.1546 & 0.1726 & 0.1524 & 0.1226 & 0.0931 & 0.0673 \\
 & $1-4\eta$ & 0.1849 & 0.0310 & 0.2455 & 0.7047 & 1.4036 & 2.3824 & 3.7419 & 5.5492 \\
 \addlinespace[4pt]
 \multirow{2}{*}{5:5, random} & $0.9$ & 0.3162 & 0.0083 & 0.1343 & 0.1539 & 0.1310 & 0.1002 & 0.0746 & 0.0534 \\
  & $1-4\eta$ & 0.2428 & 0.1064 & 0.1862 & 0.6293 & 1.2251 & 1.9919 & 3.0391 & 4.4557 \\
 \bottomrule
 \end{tabular*}
 }
 \endgroup
 \setlength{\abovecaptionskip}{7pt}
 \setlength{\belowcaptionskip}{0pt}
 \caption{SNR of the objective difference between independent EMA
 and no-EMA runs. Each ratio gives the numbers of reward-one and reward-zero arms, respectively.}
 \label{tab:bandit-snr}
\end{table}

\subsection{Large language model experiments}
\label{app:llm-experiments}

\noindent\textbf{Settings.} We use Qwen3-4B-Base and Qwen2.5-32B models using the verl training framework. The parameters are set as follows.

\begin{table}[htbp]
\centering
\label{tab:llm-training-settings}
\small
\setlength{\tabcolsep}{3pt}
\renewcommand{\arraystretch}{1.15}
\begin{tabular*}{\textwidth}{@{\extracolsep{\fill}}llll@{}}
\toprule
\textbf{Parameter} & \textbf{Value} & \textbf{Parameter} & \textbf{Value} \\
\midrule
Training dataset & DAPO-17k (English) & Generations per prompt & 8 \\
Prompts per batch & 16 & Updates per RL step & 1 \\
Max prompt length & 2,048 & Max response length & 4,096 \\
Rollout temperature & 1.0 & Rollout top-$p$ & 1.0 \\
Clip ratio & 0.2 & Gradient clipping norm & 1.0 \\
KL coefficient & 0.01 & Training generation seeds & 1001, 2001, 3001 \\
\bottomrule
\end{tabular*}
\caption{Shared training settings for the LLM experiments.}
\end{table}

\noindent\textbf{Prompts and rewards.}
We use the original prompts provided in the DAPO-17k English dataset.
As in DAPO's implementation, we restrict answer extraction to the final 300 characters of each response.
A response receives a reward of $1$ if the extracted answer matches the ground truth, and $0$ otherwise.
The same rule applies to responses truncated at the length limit; truncation alone does not force a zero reward.

\vspace{5pt}
\noindent\textbf{Validation.}
To assess the optimization performance, we evaluate on held-out data drawn from the same distribution as the training data. We construct this validation set by randomly holding out 1,000 problems from the DAPO-17k English dataset and excluding them from training. Evaluation is performed every 20 optimization steps, with an additional evaluation at step 150 for runs ending at that step.
For each problem, we sample eight responses at a temperature of 0.7 and top-$p$ of 1.0, with a maximum response length of 4,096 tokens. Each response receives a binary correctness score from the answer verifier. We report the average score over all 8,000 responses, equivalently the mean per-problem accuracy. All runs use the same held-out problems.

\end{document}